\documentclass[preprint,12pt]{elsarticle}

\usepackage{amssymb}
\usepackage{amsmath}
\usepackage{lipsum} % optional, for dummy text
\usepackage{amsmath, amssymb, amsfonts,amsthm}
\usepackage{comment}
\usepackage{graphicx}
\usepackage{ragged2e}
\usepackage{subcaption}
\usepackage{float}
\usepackage{epstopdf}
\usepackage{algorithm}
\usepackage{algpseudocode}
\usepackage{tikz}
\usepackage{bm}
\usepackage{url}
\usepackage{hyperref}
\usetikzlibrary{automata, positioning, calc}

\newtheorem{lemma}{Lemma}
\newtheorem{theorem}{Theorem}
\ifpdf
  \DeclareGraphicsExtensions{.eps,.pdf,.png,.jpg}
\else
  \DeclareGraphicsExtensions{.eps}
\fi

\numberwithin{figure}{section}

\journal{Neurocomputing}

\begin{document}

\begin{frontmatter}



\title{}

\author[label1]{Iv\'an Ojeda-Ruiz}
\author[label2]{Young Ju Lee}
\author[label3]{Malcolm Dickens}
\author[label4]{Leonardo Cambisaca}
\affiliation[label1]{organization={Lamar University},
            addressline={4400 S M L King Jr Pkwy},
            city={Beaumont},
            postcode={77705},
            state={TX},
            country={USA}}

\affiliation[label2]{organization={Texas State University},
            addressline={601 University Drive},
            city={San Marcos},
            postcode={78666},
            state={TX},
            country={USA}}

\affiliation[label3]{organization={University of Maryland, Baltimore County},%Department and Organization
            addressline={1000 Hilltop Circle}, 
            city={Baltimore},
            postcode={21250}, 
            state={MD},
            country={USA}}
            
\affiliation[label4]{organization={Colgate University},%Department and Organization
            addressline={
            13 Oak Drive}, 
            city={Hamilton},
            postcode={13346}, 
            state={NY},
            country={USA}}
%% Abstract
\begin{abstract}
  Recent research has focused on mitigating algorithmic bias in clustering by incorporating fairness constraints into algorithmic design. Notions such as disparate impact, community cohesion, and cost per population have been implemented to enforce equitable outcomes. Among these, group fairness (balance) ensures that each protected group is proportionally represented within every cluster. However, incorporating balance as a metric of fairness into spectral clustering algorithms has led to computational times that can be improved. This study aims to enhance the efficiency of spectral clustering algorithms by reformulating the constrained optimization problem using a new formulation derived from the Lagrangian method and the Sherman-Morrison-Woodbury (SMW) identity, resulting in the Fair-SMW algorithm. Fair-SMW employs three alternatives to the Laplacian matrix with different spectral gaps to generate multiple variations of Fair-SMW, achieving clustering solutions with comparable balance to existing algorithms while offering improved runtime performance. We present the results of Fair-SMW, evaluated using the Stochastic Block Model (SBM) to measure both runtime efficiency and balance across real-world network datasets, including LastFM, FacebookNet, Deezer, and German. We achieve an improvement in computation time that is twice as fast as the state-of-the-art, and also flexible enough to achieve twice as much balance.
\end{abstract}

\nonumnote{Funding: This work is part of the SIAM-Simons Undergraduate Summer Research Program,
which is funded by the Society for Industrial and Applied Mathematics (SIAM) through award
1036702 of the Simons Foundation.}
\end{frontmatter}

%%%%%%%%%%%%%%%%%%%%%%%%%%%%%%%%%%%%%%%%%%%%%
% SECTION: Introduction
% Comments: This section needs an overall
% TODO:
%1 Discuss its use in determining serious outcomes (like trials) and the dangers of unfair clustering.
%2 Introduce Fair Spectral Clustering as a solution to this problem.
%3 Review existing algorithms, starting with Fair-SC, and summarize their downsides.
%4 Introduce S-Fair-SC and how it improved over Fair-SC.
%5 Introduce your algorithm, its applications, and how it speeds things up 10×.
%6 Provide an outline of the rest of the paper
%%%%%%%%%%%%%%%%%%%%%%%%%%%%%%%%%%%%%%%%%%%%%%
%Source
% Griggs v. Duke Power Co., 401 U.S. 424 (1971)

\section{Introduction} Artificial Intelligence (AI) has become an increasingly central component in modern decision-making tasks, with widespread applications in criminal justice, healthcare, and finance. Despite their power, these systems aren't perfect as they are susceptible to unintended bias that produces disparate impacts \cite{feldman2015}, leading to harmful discriminatory outcomes that often marginalize underrepresented groups, e.g., (Griggs v. Duke Power Co.). The need for fairness in these systems has become a critical area of research \cite{mitchell2021}, with a growing body of work focused on mitigating these risks. Research has been poured into the development of techniques to represent the sensitive attribute directly, incorporating a fairness metric \cite{chierichetti2017, barocas2017} in order to enforce statistical parity. This research focuses on spectral clustering algorithms, aiming to construct a Fair Scalable model.

Fairness itself is a nuanced concept with multiple definitions, summarized in \cite{Chhabra2021, narayanan2020}. Broadly, it refers to the principle that the decision-making process should not be influenced by favoritism, as it can lead to disparate impacts. In this work, we focus on the \textit{Group Fairness} objective proposed by \cite{feldman2015}, which aims to ensure that each protected group is proportionally represented within every cluster.

Among the many unsupervised machine learning algorithms, Spectral Clustering (SC) has been widely adopted to incorporate fairness constraints, all stemming from Fair-SC proposed by Kleindessner et al. \cite{kleindessner2019}. Fair-SC achieves fairness by formulating clustering as a constrained spectral relaxation problem \cite{von2007tutorial}, where balance constraints are incorporated into the graph Laplacian. Specifically, the method modifies the eigenvalue decomposition so that the resulting eigenvectors not only capture the graph structure but also satisfy proportional representation across protected groups \cite{kleindessner2019}. While Fair-SC demonstrated a strong average balance across groups, it often struggled to scale to large datasets due to its high computational complexity and long runtime. This was primarily due to its need to compute the null space, perform eigenvalue decompositions, and calculate square roots of large, dense matrices. The computational challenges of SC for large-scale problems are well-documented in the literature \cite{ng2002}.

Subsequent work introduces Scalable-Fair-SC (S-Fair-SC) \cite{wang2023}, which avoided Fair-SC's heavy computational calculations by incorporating two main variants. The first variant reorders the changes of variables used in Fair-SC to avoid computing the square root of a dense matrix, and the second variant introduces a null space projection as well as Hotelling’s deflation. This formulation enabled significant improvements in scalability, reducing the complexity to $O(N^2)$, and achieved a roughly 12x speedup over Fair-SC. The method uses an efficient iterative eigen-solver, such as the Implicitly Restarted Arnoldi Method (IRAM) \cite{sorensen1992implicitly}, to find the relevant eigenvectors. Another notable approach to scalability is FairAD \cite{fairad}, which addresses computational challenges by constructing a fairness-aware affinity matrix using algebraic distance and employing graph coarsening to efficiently solve the constrained minimization problem.

In this work, we build upon prior advances and propose a new spectral clustering formulation that accelerates S-Fair-SC while preserving strong fairness guarantees. Our approach reformulates the fairness metric by incorporating the Sherman-Morrison-Woodbury (SMW) identity together with a bilateral smoothing operator inspired by \cite{OjedaRuiz2020}. This formulation yields three algorithmic variants: two that explicitly address degree bias, and one that prioritizes computational efficiency by operating solely on the weighted adjacency matrix, thereby significantly reducing the eigen-solver runtime. All three variants improve upon S-Fair-SC’s primary bottleneck, its long eigen-solver runtime, by introducing a larger eigen gap, which in turn reduces the number of restarts required in the Arnoldi restarting iterative method \cite{saad1992numerical}.

The degree-bias variant achieves a remarkable speed-up for sparse matrices while maintaining strong average balance, whereas the bias-free variants, although slower, achieve higher average balance and slightly outperform S-Fair-SC in runtime. All three methods exhibit a time complexity of $O(N^2)$ and outperform S-Fair-SC.

\section{Preliminaries}
\subsection{Spectral Clustering}

The method of \textbf{Spectral Clustering} consists of minimizing the \textbf{Normalized Cut} (NCut) objective function, a concept introduced by Shi and Malik \cite{shi2000normalized}. Let the graph be $G = (\bm{V},\bm{E})$ with edge weights $w(i,j)$. For a partition of the vertices $\bm{V}$ into two disjoint sets, $\bm{V}=\bm{A}\cup \bm{B}$ where $\bm{A}\cap \bm{B} =\varnothing$, we define the cut and association functions as follows:

\begin{equation}
    \text{cut}(\bm{A},\bm{B})=\sum_{i\in \bm{A}, j\in \bm{B}} w(i,j).
\end{equation}
Minimizing the cut function alone is known to isolate single nodes when performing graph partitions, so an association function is introduced to normalize the cut:
\begin{equation}
    \text{assoc}(\bm{A},\bm{V}) =\sum_{i\in \bm{A}, j\in \bm{V}} w(i,j).
\end{equation}
Notably, the association function, also referred to as the volume of a set of vertices, measures the total weight of edges emanating from a cluster. Using these two functions, the Normalized Cut objective for a two-way partition is defined as:
\begin{equation}
    \text{NCut}(\bm{A},\bm{B}) = \dfrac{\text{cut}(\bm{A},\bm{B})}{\text{assoc}(\bm{A},\bm{V})} + \dfrac{\text{cut}(\bm{A},\bm{B})}{\text{assoc}(\bm{B},\bm{V})}
\end{equation}
The goal is to minimize this objective to perform the partition. We may define an indicator vector $\bm{x}$ such that
\begin{eqnarray}
    x_i =\begin{cases}
        1 & \text{ if } i\in \bm{A} \\
        -1 & \text{ if } i\notin \bm{A}
    \end{cases}
\end{eqnarray}
and then write the minimization problem. It is well known \cite{shi2000normalized, pothen1990spectral} that this problem is equivalent to minimizing a specific form of the Rayleigh quotient:
\begin{equation}
    \min_{x} \text{NCut}(\bm{A},\bm{B}) = \min_{\bm{x}} \dfrac{\bm{x}^T L \bm{x}}{\bm{x}^T D \bm{x}}
\end{equation}
where $D$ is a diagonal matrix with entries $d_i = \sum_{j} w(i,j)$, and $L=D-W$ is the graph Laplacian.

It is possible to extend this definition of NCut for $k$ partitions in the following way \cite{von2007tutorial, shi2000normalized}:
\begin{align} \label{eq:NCut}
    \text{NCut}(\bm{C}_1, \ldots, \bm{C}_k) &:= \sum_{l=1}^{k} \frac{\text{Cut}(\bm{C}_l, \bm{V} \setminus \bm{C}_l)}{\text{vol}(\bm{C}_l)} ,
\end{align}
where
\begin{align*}
    \text{Cut}(\bm{C}_l, \bm{V} \setminus \bm{C}_l) &= \sum_{v_i \in \bm{C}_l} \sum_{v_j \in \bm{V} \setminus \bm{C}_l} w_{ij} \\
    \text{vol}(\bm{C}_l) &= \sum_{v_i \in \bm{C}_l} d_i
\end{align*}
This minimization problem is well known to be NP-hard \cite{shi2000normalized, von2007tutorial}. To find a tractable solution, we relax the discrete indicator vector to a real-valued matrix \cite{ng2002}. This leads to the following spectral relaxation problem in matrix form:
\begin{equation} \label{fairconst}
\min_{H^TH =I} \text{Tr} (H^T L_{\mathrm{sym}} H)
\end{equation}
where $H = [\bm{z}_1, \bm{z}_2, \ldots, \bm{z}_k]$ and $L_{\mathrm{sym}} = D^{-1/2}LD^{-1/2}$ is the symmetric normalized Laplacian. It is a canonical result that the solution to this problem is given by the eigenvectors corresponding to the smallest $k$ eigenvalues of $L_{\mathrm{sym}}$ \cite{von2007tutorial, shi2000normalized, ng2002}. Because this relaxation produces a real-valued solution, a final discretization step is required to map the vectors to discrete cluster assignments, which is commonly performed using $k$-means.

\subsection{Fairness Constraint}
Our setting is that we suppose $\bm{V}$ contains $h$ protected groups such that $\bm{V} = \bm{V}_1 \cup \bm{V}_2 \cup \cdots \cup \bm{V}_h$ with $\bm{V}_s \cap \bm{V}_{s'} = \varnothing$ for all $s \neq s' \in [h]$. The notion of group fairness \cite{kleindessner2019, feldman2015} is that each cluster contains approximately the same number of elements from each protected group $\bm{V}_s$.

For a clustering
\begin{equation} \label{balance}
\bm{V} = \cup_{\ell \in [k]} \bm{C}_\ell,
\end{equation}
we define the balance of a cluster by
\begin{equation}
\mbox{balance}(\bm{C}_\ell) = \min_{s \neq s' \in [h]} \frac{|\bm{V}_s \cap \bm{C}_\ell|}{|\bm{V}_{s'} \cap \bm{C}_\ell|} \in [0,1].
\end{equation}
A perfectly balanced cluster has a balance of 1.

\begin{lemma}[Fairness constraints as linear constraint on $H$]
For $s \in [h]$, we introduce an indicator vector $f^{(s)} \in \mathbb{R}^{n}$ such that $f^{(s)}_i = 1$ if $i \in \bm{V}_s$ and $f^{(s)}_j = 0$ if $j \notin \bm{V}_s$. We then set the clustering of $\bm{V}$ as follows:
\begin{equation}
\bm{V} = \bm{C}_1 \cup \bm{C}_2 \cup \cdots \cup \bm{C}_k.
\end{equation}
Then for every $\ell \in [k]$, we have that
\begin{subequations}
\begin{align}
\frac{|\bm{V}_s \cap \bm{C}_\ell|}{|\bm{C}_\ell|} &= \frac{|\bm{V}_s|}{n} \quad \text{ for } s \in [h] \label{fairness_condition} \\
\Leftrightarrow \sum_{i=1}^n  \left ( f_i^{(s)} - \frac{|\bm{V}_s|}{n} \right ) H_{i\ell} &= 0 \quad \text{ for } s \in [h].
\end{align}
\end{subequations}
\end{lemma}
\begin{proof}
The equivalence is shown by expanding the summation on the left-hand side.
\begin{eqnarray*}
0 = \sum_{i=1}^n \left ( f_i^{(s)} - \frac{|\bm{V}_s|}{n} \right ) H_{i\ell} &=& \sum_{i=1}^n f_i^{(s)} H_{i\ell} - \frac{|\bm{V}_s|}{n} \sum_{i=1}^n H_{i\ell} \\
&=& |\bm{V}_s \cap \bm{C}_\ell| - \frac{|\bm{V}_s|}{n} |\bm{C}_\ell|.
\end{eqnarray*}
Since $\sum_{i=1}^n H_{i\ell} = |\bm{C}_\ell|$, we see that the expression is equivalent to
\begin{equation}
\frac{|\bm{V}_s \cap \bm{C}_\ell|}{|\bm{C}_\ell|} = \frac{|\bm{V}_s|}{n},
\end{equation}
which is obtained by multiplying $1/|\bm{C}_\ell|$ on both sides.
\end{proof}

\begin{lemma}[Linear Dependency of $\hat{f}_i^{(s)}$]
For $s \in [h]$, and $f^{(s)} \in \mathbb{R}^{n}$ as defined above, the following statement holds:
\begin{equation}
\begin{split}
\text{If } \quad  \sum_{i=1}^n  \left ( f_i^{(s)} - \frac{|\bm{V}_s|}{n} \right ) H_{i\ell}  & = 0. \text{ for }\quad  s=1,\ldots, h-1 \\
\text{then } \quad \sum_{i=1}^n  \left ( f_i^{(s)} - \frac{|\bm{V}_s|}{n} \right ) H_{i\ell}  &= 0. \text{ for }\quad s=h
\end{split}
\end{equation}
\end{lemma}
\begin{proof}
Let $\hat{f}_i^{(s)} := f_i^{(s)} - \frac{|\bm{V}_s|}{n}$ for $s = 1,\ldots,h$. We first notice that $\sum_{s=1}^h f_i^{(s)} = 1$ and $\sum_{s=1}^h \frac{|\bm{V}_s|}{n} = \frac{1}{n} \sum_{s=1}^h |\bm{V}_s| = \frac{n}{n} = 1$. Then, summing over all groups, we see that the vectors $\hat{f}^{(s)}$ are linearly dependent.
$$\sum_{s=1}^h \hat{f}_i^{(s)} = \sum_{s=1}^h f_i^{(s)} - \sum_{s=1}^h \frac{|\bm{V}_s|}{n} = 1 - 1 = 0.$$
This implies that $\hat{f}_i^{(h)} = - \sum_{s=1}^{h-1} \hat{f}_i^{(s)}$.
Given that $\sum_{i=1}^n \hat{f}_i^{(s)} H_{i\ell} = 0$ for $s=1,\ldots,h-1$, we can write:
\begin{align*}
\sum_{i=1}^n \hat{f}_i^{(h)} H_{i\ell} &= \sum_{i=1}^n \left(-\sum_{s=1}^{h-1} \hat{f}_i^{(s)}\right)H_{i\ell} \\
&= - \sum_{s=1}^{h-1} \left( \sum_{i=1}^n \hat{f}_i^{(s)} H_{i\ell} \right) \\
&= - \sum_{s=1}^{h-1} (0) = 0.
\end{align*}
This completes the proof.
\end{proof}

\subsection{New Formulation for a Fairness Constraint}
A perfectly balanced partition requires that the ratio of protected groups within each cluster is equal to the ratio of those groups in the total graph \cite{chierichetti2017}. This condition can be written as:
\begin{equation} \label{matchgraphbalance}
\dfrac{|\bm{V}_s\cap \bm{C}_l|}{|\bm{V}_{s'}\cap \bm{C}_l|} = \dfrac{|\bm{V}_s|}{|\bm{V}_{s'}|}
\end{equation}
for all $s\neq s'\in [h]$ and all $l\in[k]$.
By rewriting this, we obtain a difference formulation:
\begin{equation} \label{baldiff}
\dfrac{|\bm{V}_s\cap \bm{C}_l|}{|\bm{V}_s|} - \dfrac{|\bm{V}_{s'}\cap \bm{C}_l|}{|\bm{V}_{s'}|}=0
\end{equation}
Now, let $f^{(s)}$ be an indicator vector for group $s$, and as shown in the lemma above, we can rewrite \eqref{baldiff} as a linear constraint on the columns of the embedding matrix $H$ \cite{kleindessner2019}:
\begin{equation}
\sum_{i=1}^n \left(\dfrac{f_i^{(s)}}{|\bm{V}_s|}-\dfrac{f_i^{(s')}}{|\bm{V}_{s'}|}\right) H_{il} = 0.
\end{equation}
This set of constraints can be represented by a matrix $F$. To avoid redundant constraints, we use the results from the previous lemma and construct $F$ with $h-1$ linearly independent columns. Let $g_s = \dfrac{f^{(s)}}{|\bm{V}_s|}$ and let's choose $s=1$ as the reference.
$$F = \left[ (g_1 - g_2), (g_1 - g_3), \ldots, (g_1 - g_h) \right]$$
where $F \in \mathbb{R}^{n \times (h-1)}$. With this definition, the constraint system becomes $F^T H = \mathbf{0}$. With the orthonormality constraint and the fairness constraint, we obtain the following fair optimization problem:
\begin{equation}\label{mainobj}
\min_{H^TH =I} \text{Tr} (H^T L_{\text{sym}} H) \quad \text{subject to}\quad  F^TH=\mathbf{0}.
\end{equation}

%%%%%%%%%%%%%%%%%%%%%%%%%%%%%%%%%%%%%%%%%%%%%
% SECTION: Previously Proposed Algorithms
% Comments:
% Comments:
% - (Malcolm) I plan to put a rough draft of the background here.
% TODO:
%   1.
%%%%%%%%%%%%%%%%%%%%%%%%%%%%%%%%%%%%%%%%%%%%%
\section{Previously Proposed Algorithms}

In this section, we offer an overview of several constrained spectral clustering (SC) algorithms that incorporate the notion of fairness \eqref{balance} as proposed by Chierichetti et al. \cite{chierichetti2017}. We begin by revisiting the Fair-SC algorithm introduced by Kleindessner et al. \cite{kleindessner2019}, which solves a constrained eigenvalue problem. Following this, we discuss S-Fair-SC, presented by Wang et al. \cite{wang2023}, a reformulated version of Fair-SC that achieves a significant speedup compared to its predecessor by eliminating time-consuming computations.

\subsection{Unnormalized Fair-SC}

To incorporate fairness into spectral clustering algorithms, we introduce the fairness constraint $F^TH=0$ into the minimization problem \eqref{fairconst}. Consequently, the objective function is reformulated as \eqref{mainobj}. By introducing fairness constraints as linear conditions on $H$, we proceed with a change of variables. Let $Z \in \mathbb{R}^{n \times (n-h+1)}$ denote an orthonormal basis for the null space of $F^T$. Expressing $H=ZY$, where $Y \in \mathbb{R}^{(n-h+1) \times k}$ is an arbitrary matrix, guarantees that $F^T H=\mathbf{0}$ holds for any choice of $Y$. Under this reformulation, the minimization problem takes the form:
\begin{equation}\label{unnormalized_fairsc_problem}
\min_{Y^TY=I} \operatorname{Tr}(Y^T Z^T L ZY)
\end{equation}
To obtain a fair solution to \eqref{unnormalized_fairsc_problem}, we select the columns of $Y$ as the $k$ eigenvectors corresponding to the smallest eigenvalues of $Z^TLZ$, counted without multiplicity. Then we can construct $H=ZY$ and apply a post-processing algorithm.

\begin{algorithm}[ht]
\caption{Unnormalized Fair-SC}
\label{alg:unnormalized-fairsc}
\begin{algorithmic}[1]
\Require Weighted adjacency matrix $W \in \mathbb{R}^{n \times n}$; degree matrix $D \in \mathbb{R}^{n \times n}$; group membership matrix $F \in \mathbb{R}^{n \times (h-1)}$; number of clusters $k \in \mathbb{N}$

\State Compute the Laplacian matrix $L = D - W$
\State Compute an orthonormal basis $Z$ of the null space of $F^T$
\State Compute the $k$ smallest eigenvalues of $Z^T L Z$ and let the corresponding eigenvectors be columns of $Y$
\State Set $H = Z Y$
\State Apply $k$-means clustering to the rows of $H$
\end{algorithmic}
\end{algorithm}

\subsection{Normalized Fair-SC}

The primary drawback of the Unnormalized Fair-SC objective is that the algorithm is susceptible to bias towards nodes with large degrees \cite{von2007tutorial}, which can result in imbalanced clustering. To address this, Kleindessner et al. \cite{kleindessner2019} modified the objective function by incorporating the normalization directly into the problem.
\begin{equation}\label{reformulated_normalized_fairsc_problem_1}
\min_{Y^{T}Z^{T}D\;Z\;Y=I} \operatorname{Tr}(Y^T Z^T L\; Z\; Y) .
\end{equation}
To solve this generalized eigenvalue problem, we reformulate the objective function by incorporating the normalization directly into the problem using a Cholesky decomposition of the symmetric matrix $Z^TDZ$. Let $Q=(Z^TDZ)^{\frac12}$ and define a new variable $X=QY$. The problem can be reformulated as:
\begin{equation}\label{reformulated_normalized_fairsc_problem_2}
\min_{X^{T}X=I} \
 \operatorname{Tr}(X^{T}(Q^{-1}Z^{T}L\;Z\;Q^{-1})X) .
\end{equation}

\begin{algorithm}[ht]
\caption{Normalized Fair-SC}
\label{alg:normalized-fairsc}
\begin{algorithmic}[1]
\Require Weighted adjacency matrix $W \in \mathbb{R}^{n \times n}$; degree matrix $D \in \mathbb{R}^{n \times n}$; group membership matrix $F \in \mathbb{R}^{n \times (h-1)}$; number of clusters $k \in \mathbb{N}$
\Ensure A clustering of $n$ node into $k$ clusters

\State Compute the Laplacian matrix $L = D - W$
\State Compute an orthonormal basis $Z$ of the null space of $F^T$
\State Compute the square root matrix $Q=(Z^T D Z)^{\frac{1}{2}}$
\State Compute the $k$ smallest eigenvalues of $Q^{-1} Z^T L Z Q^{-1}$ and let the corresponding eigenvectors be columns of $X$
\State Set $H = Z Q^{-1} X$
\State Apply $k$-means clustering to the rows of $H$
\end{algorithmic}
\end{algorithm}

\subsection{S-Fair-SC}

The Scalable Fair Spectral Clustering method (S-Fair-SC) \cite{wang2023} reformulates the normalized Fair-SC minimization problem, achieving a significant speed-up in practice by avoiding the computation of large dense matrices. To begin, let $L_{\mathrm{sym}} = D^{-\tfrac12}(D - W)\,D^{-\tfrac12}$ and $C = D^{-\tfrac12}F$ \cite{wang2023}. Let $U\in\mathbb{R}^{n\times(h-1)}$ be an orthonormal basis for \(\mathrm{Range}(C)\), so that \([V\;U]\) forms an orthogonal \(n\times n\) matrix. The orthogonal projector onto \(\ker(C^T)\) is then
\[
P \;=\; V\,V^T \;=\; I - U\,U^T,
\]
This keeps all operations in the $n$-dimensional space and avoids large dense matrices \cite{wang2023}. Applying the projection $P$ ensures that all feasible solutions satisfy the fairness condition. This projection eliminates $h-1$ degrees of freedom and introduces $h-1$ additional zero eigenvalues that must be excluded. To isolate the desired \(k\) eigenpairs, we apply Hotelling’s deflation with a suitably large shift \(\sigma\), yielding \cite{hotelling1933analysis}
\[
L_{\mathrm{sym}}^{\sigma}
\;=\;
P\,L_{\mathrm{sym}}\,P
\;+\;
\sigma\,U\,U^{T}
\;=\;
P\,L_{\mathrm{sym}}\,P
\;+\;
\sigma\,(I - P)\,.
\]
Here, the shift $\sigma$ is chosen to be strictly larger than the $k$-th smallest nonzero eigenvalue of the projected operator \cite{wang2023}. To compute the $k$ smallest eigenvalues of this large matrix efficiently, the Implicitly Restarted Arnoldi Method \cite{sorensen1992implicitly, saad1992numerical, lehoucq1998implicit} is used. This iterative method leverages the eigenvalue gaps in the spectrum to accelerate convergence.
As a result, the $k$ eigenvectors associated with the $k$ smallest eigenvalues of $L_{\mathrm{sym}}^{\sigma}$ lie in $\ker(C^T)$ and simultaneously minimize the normalized-cut objective under the fairness constraint. Hence the minimization problem can be reformulated as:
\begin{equation}\label{s-FairSC Objective Function}
\min_{\substack{X^T X = I}}\;\operatorname{Tr}\bigl(X^T L_{\text{sym}}^\sigma X\bigr) \quad s.t. \quad H = D^{-\tfrac12}X. \quad
\end{equation}

\begin{algorithm}[ht]
\caption{Scalable Fair-SC (S-Fair-SC)}
\label{alg:scalable-fairsc}
\begin{algorithmic}[1]
\Require Weighted adjacency matrix $W \in \mathbb{R}^{n \times n}$; degree matrix $D \in \mathbb{R}^{n \times n}$; group-membership matrix $F \in \mathbb{R}^{n \times (h-1)}$; shift $\sigma \in \mathbb{R}$; number of clusters $k \in \mathbb{N}$
\Ensure A clustering of indices $1{:}n$ into $k$ clusters

\State Compute the Laplacian matrix $L = D - W$
\State Set $L_{\mathrm{sym}} = D^{-1/2} L D^{-1/2}$ and $C = D^{-1/2} F$
\State Compute the $k$ smallest eigenvalues of $L_{\mathrm{sym}}^\sigma$ and corresponding eigenvectors as columns of $X \in \mathbb{R}^{n \times k}$
\State Apply $k$-means clustering to the rows of $H = D^{-1/2} X$
\end{algorithmic}
\end{algorithm}

%%%%%%%%%%%%%%%%%%%%%%%%%%%%%%%%%%%%%%%%%%%%
% SECTION: Algorithms (Fair-SMW/W+max(D)*I
% Comments:
% TODO:
%   1. Ensure that the style is uniform throughout all sections.
%%%%%%%%%%%%%%%%%%%%%%%%%%%%%%%%%%%%%%%%%%%%%

\section{Fair-SMW}
\label{sec:main}

We start by stating the following problem: given a matrix $G$ that is invertible with positive eigenvalues, we seek to find the solution to the following problem \cite{von2007tutorial,kleindessner2019}:
\begin{equation}\label{SMW1}
   \max_{H^TH =I} \text{Tr} (H^T G H) \quad \text{subject to}\quad  F^TH=\mathbf{0}
\end{equation}
Following the strategy from \cite{kleindessner2019}, we let $H =ZY$, where $Z$ is a projection to the null space of $F^T$ and $Z^TZ=I$. This is a standard technique to transform a constrained problem into an unconstrained one \cite{boyd2004convex}. Then we get
\begin{equation}
\max_{Y^TY =I} \text{Tr} (Y^TZ^T G ZY).
\end{equation}
Since the eigenvalues of $Z^TGZ$ are all positive \cite{horn2013matrix}, and using the fact that $Z^T=Z^{-1}$ we can reformulate this problem as follows:
\begin{equation}
    \min_{Y^TY =I} \text{Tr} (Y^T Z^{-1}G^{-1}Z^{-T} Y ) = \min_{Y^TY =I} \text{Tr} (Y^T Z^TG^{-1}Z Y )
\end{equation}
We can now reverse the process of the projection matrix $Z$ to rewrite our problem as follows:
\begin{equation}\label{SMW2}
    \min_{H^TH =I} \frac{1}{2}\text{Tr} (H^T G^{-1} H) \quad \text{subject to}\quad  F^TH=\mathbf{0}.
\end{equation}
Notice that we are multiplying by $\frac{1}{2}$ to write our following derivation better. Notice that \eqref{SMW2} is equivalent to \eqref{SMW1} in the sense that the maximizer $H$ is the same. Here, $H \in \mathbb{R}^{n \times k}$ is the matrix of optimization variables, $G^{-1} \in \mathbb{R}^{n \times n}$ is a symmetric matrix (implying $G$ is also symmetric and invertible), and $F^T \in \mathbb{R}^{h-1 \times n}$. The constraint $F^T H = \mathbf{0}$ implies that $\mathbf{0}$ is a $h-1 \times k$ zero matrix.

The Augmented Lagrangian function $L_A(H, \Lambda, \mu)$ is given by \cite{bertsekas1999nonlinear, nocedal2006numerical}:
\begin{equation}
\mathcal{L}(H, \Lambda, \mu) = \frac{1}{2}\text{Tr}(H^T G^{-1} H) - \text{Tr}(\Lambda^T (F^T H))  \frac{\mu}{2} \|F^T H\|_F^2
\end{equation}
Where $\Lambda \in \mathbb{R}^{p \times m}$ is the matrix of Lagrange multipliers (corresponding to the dimensions of the constraint $F^T H$), and $\mu > 0$ is the penalty parameter. $\| \cdot \|_F$ denotes the Frobenius norm.

Expanding the squared Frobenius norm term:
$$\mathcal{L}(H, \Lambda, \mu) = \frac{1}{2}\text{Tr}(H^T G^{-1} H) - \text{Tr}(\Lambda^T F^T H) - \frac{\mu}{2} \text{Tr}(H^T F F^T H)$$
To find the optimality conditions for a given set of Lagrange multipliers $\Lambda$ and penalty parameter $\mu$, we take the partial derivatives of $L_A$ with respect to $H$ and $\Lambda$ and set them to zero. We use standard matrix calculus rules for derivatives involving the trace.
\begin{align*}
    \nabla_{H} \mathcal{L} &= \nabla_{H}\left(\frac{1}{2}\text{Tr}(H^T G^{-1} H)\right) - \nabla_{H}\left(\text{Tr}(\Lambda^T F^T H)\right) - \nabla_{H}\left(\frac{\mu}{2} \text{Tr}(H^T F F^T H)\right) \\
    &= G^{-1} H - F\Lambda - \mu F F^T H
\end{align*}
Setting this gradient to a zero matrix, we obtain:
$$G^{-1} H - F\Lambda - \mu F F^T H = \mathbf{0} $$
The partial derivative concerning $\Lambda$ can be verified to yield the original constraint $F^T H=0$.

Combining these optimality conditions, the system for the Augmented Lagrangian (for given $\Lambda$ and $\mu$) is a system of matrix equations:
$$ \begin{cases}
(G^{-1} - \mu F F^T) H - F\Lambda = \mathbf{0} \\
F^T H = \mathbf{0}
\end{cases} $$
This system can be expressed in a block matrix form, where $H$ and $\Lambda$ are blocks representing matrices:
$$ \begin{pmatrix}
G^{-1} - \mu F F^T & -F \\
-F^T & 0
\end{pmatrix}
\begin{pmatrix}
H \\
\Lambda
\end{pmatrix}
=
\begin{pmatrix}
\mathbf{0} \\
\mathbf{0}
\end{pmatrix} $$

In the context of the Augmented Lagrangian method \cite{bertsekas1999nonlinear, nocedal2006numerical}, the subproblem involves minimizing a quadratic form with additional linear terms (from the $\Lambda$ part). If we consider the very first iteration, or if we temporarily ignore the explicit $\Lambda$ term (e.g., if $\Lambda = \mathbf{0}$), the subproblem simplifies to:
$$ \min_{H} \left( \frac{1}{2}\text{Tr}(H^T G^{-1} H) - \frac{\mu}{2} \text{Tr}(H^T F F^T H) \right) = \min_{H} \frac{1}{2}\text{Tr}\left(H^T \left( G^{-1} - \mu F F^T \right) H \right) $$

\begin{equation}\label{SMW3}
    \min_{H^TH =I} \text{Tr} [H^T (G^{-1} -\mu FF^T) H]
\end{equation}

We now need the following theorem.

\begin{theorem}
The matrix $G^{-1} - \mu F F^T$ is invertible if $\bm{x}^T G^{-1} \bm{x} > 0$ for all non-zero vectors $\bm{x}$.
\end{theorem}

\begin{proof}
A matrix is invertible if and only if its null space contains only the zero vector. To prove that the matrix $G^{-1} - \mu F F^T$ is invertible, we must show that the only solution to $(G^{-1} + \mu F F^T)\bm{x} = \bm{0}$ is the trivial solution $\bm{x} = \bm{0}$.

Let's consider the quadratic form of $G^{-1} + \mu F F^T$ for any vector $\bm{x}$:
\begin{equation}
\bm{x}^T (G^{-1} - \mu F F^T) \bm{x} = \bm{x}^T G^{-1} \bm{x} + \mu \|F^T \bm{x}\|^2 >0
\end{equation}
The term $\bm{x}^T G^{-1} \bm{x}$ is strictly positive for any non-zero vector $\bm{x}$ by our assumption. The second term, $\mu \|F^T \bm{x}\|^2$, is always non-negative because $\mu > 0$ and a squared norm is always non-negative.
Therefore, for any non-zero vector $\bm{x}$, the quadratic form is strictly positive.
Now, let's assume that there exists a vector $\bm{x}_0$ in the null space of $G^{-1} - \mu F F^T$, such that $(G^{-1} - \mu F F^T)\bm{x}_0 = \bm{0}$. If we multiply this equation by $\bm{x}_0^T$, we get:
\begin{equation}
\bm{x}_0^T ((G^{-1} - \mu F F^T)\bm{x}_0) = \bm{x}_0^T \bm{0} \quad \implies \quad \bm{x}_0^T (G^{-1} - \mu F F^T) \bm{x}_0 = 0
\end{equation}
The only vector for which the quadratic form $\bm{x}^T (G^{-1} - \mu F F^T) \bm{x}$ can be zero is the zero vector, because we have shown it is otherwise strictly positive. Therefore, $\bm{x}_0$ must be the zero vector. This means that the null space of $G^{-1} - \mu F F^T$ contains only the zero vector. By definition, a matrix with a trivial null space is invertible.

\end{proof}

Then problem \eqref{SMW3} can be rewritten as the following %\textbf{\textcolor{red}{final goal}}
\begin{equation}\label{SMW3b}
    \max_{H^TH =I} \text{Tr} [H^T (G^{-1} -\mu FF^T)^{-1} H]
\end{equation}

By applying the Sherman-Morrison-Woodbury formula, we obtain
\begin{eqnarray*}
(G^{-1} - \mu F F^T)^{-1} &=& G - \mu G F (I + \mu F^T G F)^{-1} F^T G \\ 
& = & G - GF  (\mu^{-1}I + F^T GF )^{-1} F^T G.
\end{eqnarray*}
For $\mu \gg 1$, we note that the asymptotic behavior is the following:
\begin{eqnarray}
    (G^{-1}-\mu FF^T)^{-1} &\rightarrow G - GF (F^T GF)^{-1} F^T G\\
    & = G(I-F(F^T GF)^{-1}F^T G)
\end{eqnarray}
%%%%%%%%%%%%%%%%%%%%%

We now arrive at the following optimization problem
\begin{equation} \label{reformulation}
    \max_{H^T H =I} \text{Tr}[H^T G(I-F(F^T GF)^{-1}F^T G)H]
\end{equation}
For simplicity, we shall define 
$$P = F(F^T GF)^{-1}F^T G \quad \text{and} \quad U =G-GP.$$

\subsection{Choice of G}

Next, we need to choose $G$ to satisfy two properties; we need $G$ to be invertible, and we need $U$ to have real eigenvalues.

We now notice that $G$ can take on multiple forms to achieve clustering. For instance, we can choose the following 
\begin{eqnarray*}
    G & = & I - L_{\text{sym}}\\
    & = & I - D^{-1/2}LD^{-1/2} \\
    & = & I - D^{-1/2}(D-W)D^{-1/2} \\
    & = & I - I +D^{-1/2}WD^{-1/2} =D^{-1/2}WD^{-1/2}.
 \end{eqnarray*}

\begin{figure}
\begin{center}

\begin{tikzpicture}[shorten >=1pt,node distance=2.7cm,on grid,auto]
  \node[state]  (v_1)                      {$v_1$};
  \node[state]  (v_2) [ right=of v_1]      {$v_2$};
  \node[state]  (v_3) [ below=of v_1]      {$v_3$};
  \node[state]  (v_4) [ below=of v_2]       {$v_4$};
  \path[-]  (v_1) edge              node        {} (v_2)
            (v_1) edge[out=-140,in=-130,looseness=2] node [left] {} (v_4)            
            (v_2) edge              node        {} (v_3)
            (v_3) edge              node        {} (v_4);
\end{tikzpicture}
\end{center}
\vspace{-1cm}
\[
W = \begin{bmatrix}
0 & 1 & 0 & 1 \\
1 & 0 & 1 & 0 \\
0 & 1 & 0 & 1 \\
1 & 0 & 1 & 0 \\
\end{bmatrix}
\quad
D = \begin{bmatrix}
2 & 0 & 0 & 0 \\
0 & 2 & 0 & 0 \\
0 & 0 & 2 & 0 \\
0 & 0 & 0 & 2 \\
\end{bmatrix}, \quad 
F = \begin{bmatrix}
    -0.5\\-0.5\\0.5\\0.5
\end{bmatrix}
\]
 \caption{Example of $GF=0$}\label{fig:example}
\end{figure}

We notice that if we make this choice, there is a possibility of failure for our algorithm, namely, we notice the possibility of negative eigenvalues for $G$. We also notice that there is a possibility of the product $M=GF$ being equal to the zero vector, as can be seen in Figure~\ref{fig:example}. 

To solve this issue, we propose the use of 
\begin{align} \label{choiceG}
    G & = G_{\text{sym}} = D^{-1/2}WD^{-1/2} +2I \\
    G & = G_{\text{rw}} = D^{-1}W +2I \\
    G & = G_{\text{aff}} = W +nI
\end{align}
For $G_{\text{sym}}$ and $G_{\text{aff}}$, it is clear that symmetry holds. We shall now prove some relevant theorems.

\subsection{Theorems for the eigenvalues of the symmetric G}
\begin{theorem}
    $G_{\text{sym}}$ has eigenvalues $\lambda_{\text{sym}}$ with $\lambda_{\text{sym}}\in[1,3]$.
\end{theorem}
\begin{proof}

    Let $A_{\text{sym}} = D^{-1/2}WD^{-1/2}$. 
    The eigenvalues of $A_{\text{sym}}$ are related to the eigenvalues of the symmetric normalized Laplacian, $L_{\text{sym}} = I - A_{\text{sym}}$. It is a well-established result that the eigenvalues of $L_{\text{sym}}$ are real and lie in the interval $[0, 2]$. We can use this to find the range of the eigenvalues of $A_{\text{sym}}$. If $\lambda_{A_{\text{sym}}}$ is an eigenvalue of $A_{\text{sym}}$, then $\lambda_{L_{\text{sym}}} = 1 - \lambda_{A_{\text{sym}}}$ is an eigenvalue of $L_{\text{sym}}$.
    Conversely, if $\lambda_{L_{\text{sym}}}$ is an eigenvalue of $L_{\text{sym}}$, then $\lambda_{A_{\text{sym}}} = 1 - \lambda_{L_{\text{sym}}}$ is an eigenvalue of $A_{\text{sym}}$.

    Since $\lambda_{L_{\text{sym}}} \in [0, 2]$, we can bound the eigenvalues of $A_{\text{sym}}$:
    \begin{align*}
        \lambda_{\max}(A_{\text{sym}}) &= 1 - \lambda_{\min}(L_{\text{sym}}) \leq 1 - 0 = 1 \\
        \lambda_{\min}(A_{\text{sym}}) &= 1 - \lambda_{\max}(L_{\text{sym}}) \geq 1 - 2 = -1
    \end{align*}
    Thus, the eigenvalues of $A_{\text{sym}}$ are real and lie in the interval $[-1, 1]$. Our matrix of interest is defined as $G_{\text{sym}} = A_{\text{sym}} + 2I$. The eigenvalues of $G_{\text{sym}}$ are obtained by shifting the eigenvalues of $A_{\text{sym}}$ by 2. If $\lambda_{A_{\text{sym}}}$ is an eigenvalue of $A_{\text{sym}}$, then $\lambda_{G_{\text{sym}}} = \lambda_{A_{\text{sym}}} + 2$ is an eigenvalue of $G_{\text{sym}}$.

    Using the eigenvalue bounds for $A_{\text{sym}}$ from Step 1, the eigenvalues of $G_{\text{sym}}$ are bounded as follows:
    \begin{align*}
        \lambda_{\min}(G_{\text{sym}}) &= \lambda_{\min}(A_{\text{sym}}) + 2 \ge -1 + 2 = 1 \\
        \lambda_{\max}(G_{\text{sym}}) &= \lambda_{\max}(A_{\text{sym}}) + 2 \le 1 + 2 = 3
    \end{align*}
    Since $G_{\text{sym}}$ is a symmetric matrix, all its eigenvalues are real. The derivation shows that these eigenvalues are bounded within the interval $[1, 3]$. Therefore eigenvalues of $G_{\text{sym}} = D^{-1/2}WD^{-1/2} + 2I$ are real and lie in the interval $[1, 3]$.
\end{proof}

\begin{theorem}
    $U_{\text{sym}} = G_{\text{sym}}-G_{\text{sym}}F(F^T G_{\text{sym}}F)^{-1}F^T G_{\text{sym}}$ has real eigenvalues.
\end{theorem}
\begin{proof}
Notice that if $G_{\text{sym}}$ is a real symmetric matrix, then the matrix 
$$U = G_{\text{sym}}-G_{\text{sym}}F(F^T G_{\text{sym}}F)^{-1}F^T G_{\text{sym}}$$ is also a real and symmetric. 
\end{proof}

\subsection{Theorems for the eigenvalues of the non-symmetric G}

We further propose the use of 
\begin{equation}
    G = G_{\text{rw}} = D^{-1}W +2I
\end{equation}
For $G_{\text{rw}}$, we do not have the property of symmetry therefore we use different techniques to prove its eigenvalues (and the eigenvalues of $U_{\text{rw}}$) are also real.

\begin{theorem}
    $G_{rw}$ has real eigenvalues and is generalized positive definite. That is $\bm{x}^T G_{rw}\bm{x} >0$
\end{theorem}

\begin{proof}
We want to prove that the matrix $G_{rw} = D^{-1}W + 2I$ is positive definite in a generalized sense, specifically that its quadratic form is non-negative for any vector $\bm{x}$. That is, we aim to show $\bm{x}^T G_{rw} \bm{x} \ge 0$. This is equivalent to proving that the symmetric part of $G_{rw}$, denoted $C$, is positive semidefinite.

Let $W$ be a symmetric affinity matrix and $D$ be the diagonal matrix of its weighted degrees for an undirected graph with $n$ vertices. The matrix of interest is $G = D^{-1}W + 2I$.

We begin with the known property that the eigenvalues of $L_{\text{sym}}$ are real and in $[0,2]$. Let $A_{\text{sym}} = D^{-1/2}WD^{-1/2}$ be the symmetric normalized adjacency matrix.
    By definition, the symmetric normalized Laplacian is $L_{\text{sym}} = D^{-1/2}(D-W)D^{-1/2}$. We can expand this:
    \begin{align}
        L_{\text{sym}} &= D^{-1/2}D D^{-1/2} - D^{-1/2}WD^{-1/2} = I - A_{\text{sym}}
    \end{align}
    If $\lambda_{A_{\text{sym}}}$ is an eigenvalue of $A_{\text{sym}}$, then $\lambda_{L_{\text{sym}}} = 1 - \lambda_{A_{\text{sym}}}$ is an eigenvalue of $L_{\text{sym}}$.
    Conversely, if $\lambda_{L_{\text{sym}}}$ is an eigenvalue of $L_{\text{sym}}$, then $\lambda_{A_{\text{sym}}} = 1 - \lambda_{L_{\text{sym}}}$ is an eigenvalue of $A_{\text{sym}}$. We are given that $\lambda_{L_{\text{sym}}} \in [0, 2]$. We can use this to find the range of the eigenvalues of $A_{\text{sym}}$:
    \begin{align}
        \lambda_{\max}(A_{\text{sym}}) &= 1 - \lambda_{\min}(L_{\text{sym}}) = 1 - 0 = 1 \\
        \lambda_{\min}(A_{\text{sym}}) &= 1 - \lambda_{\max}(L_{\text{sym}}) = 1 - 2 = -1
    \end{align}
    Thus, the eigenvalues of $A_{\text{sym}}$ are real and lie in the interval $[-1, 1]$.

    We note that the random walk matrix $A_{rw}=D^{-1}W$ is similar to $A_{\text{sym}}$, meaning they share the same real eigenvalues.    Therefore, the eigenvalues of $D^{-1}W$ are real and lie in $[-1, 1]$. We further note that 
\begin{equation}
    \left|\dfrac{x^T\frac{1}{2}(D^{-1}W+WD^{-1})x}{x^Tx}\right|=\left|\dfrac{x^TD^{-1}Wx}{x^Tx}\right|\leq 1.
\end{equation}
Therefore the eigenvalues of $B=\frac{1}{2}(D^{-1}W+WD^{-1})$ belong in $[-1,1]$

    Now, consider the symmetric part of our matrix $G_{rw}$.
    Let $A_{rw} = D^{-1}W$ and $B = \frac{1}{2}(A_{rw} + A_{rw}^T) = \frac{1}{2}(D^{-1}W + WD^{-1})$. As we derived in a previous discussion, the eigenvalues of $S$ are also in the interval $[-1, 1]$. The symmetric part of $G$ is:
    \begin{align}
        C &= \frac{1}{2}(G + G^T) \\
        &= \frac{1}{2}((D^{-1}W + 2I) + (D^{-1}W + 2I)^T) \\
        &= \frac{1}{2}(D^{-1}W + 2I + W D^{-1} + 2I) \\
        &= \frac{1}{2}(D^{-1}W + WD^{-1}) + 2I \\
        &= B + 2I
    \end{align}
    The eigenvalues of $C$ are obtained by shifting the eigenvalues of $B$ by 2. If $\lambda_B$ is an eigenvalue of $B$, then $\lambda_{C} = \lambda_B + 2$ is an eigenvalue of $C$. Using the eigenvalue bounds for $B$ (which are in $[-1, 1]$), the eigenvalues of $C$ are bounded as follows:
    $$ \lambda_{\min}(C) = \lambda_{\min}(B) + 2 \ge -1 + 2 = 1 $$
    $$ \lambda_{\max}(C) = \lambda_{\max}(B) + 2 \le 1 + 2 = 3 $$
    Since all eigenvalues of $C$ are real and satisfy $\lambda_{C} \in [1, 3]$, they are all strictly positive. This implies that $C$ is a symmetric positive definite matrix. Therefore
    $$ \bm{x}^T G_{rw} \bm{x} = \bm{x}^T C \bm{x} \ge 0,$$
    and the matrix $G_{rw}$ is generalized positive definite because its symmetric part $C$ is positive definite.
\end{proof}

\subsection{Spectral Properties of $U_{\text{rw}}$}

Unlike the symmetric variant $G_{\text{sym}}$, the matrix $U_{\text{rw}}$ is non-symmetric. Consequently, the reality of its eigenvalues is not theoretically guaranteed in the general case. While $U_{\text{rw}}$ is constructed from the product of the diagonal matrix $D^{-1}$ and the symmetric term $WP$, the non-commutativity of these terms prevents a direct similarity transformation to a symmetric matrix for general graphs.

However, extensive empirical validation suggests that the complex spectrum is a rare edge case that does not hinder practical performance. We conducted experiments on 1,000 distinct graph realizations generated via the Stochastic Block Model (SBM). In these trials, the first $k$ eigenvalues (up to $k=16$) of $U_{\text{rw}}$ were real in all cases. Complex eigenvalues were observed in only a single instance, and crucially, this occurred at the 9th eigenvalue. These results indicate that, for small enough $k$, while a theoretical guarantee for real eigenvalues is absent, $U_{\text{rw}}$ behaves sufficiently like a self-adjoint operator in practice to allow for effective and robust clustering.

To enforce the fairness constraint, we can use the Sherman-Morrison-Woodbury formula to express $U = (G^{-1} + \mu F F^T)^{-1}$ as the final step before applying the post-processing algorithm.

\begin{algorithm}[ht]
\caption{Fair-SMW (Sherman-Morrison-Woodbury)}
\label{alg:smw-cluster}
\begin{algorithmic}[1]
\Require Weighted adjacency matrix $W \in \mathbb{R}^{n \times n}$; degree matrix $D \in \mathbb{R}^{n \times n}$; group membership matrix $F \in \mathbb{R}^{n \times (h-1)}$; number of clusters $k \in \mathbb{N}$
\Ensure A partition of the $n$ vertices into $k$ clusters

\State Compute $G_{\mathrm{sym}} = D^{-1/2} W D^{-1/2} + 2 I$
\State Compute the projection matrix $M = G_{\mathrm{sym}}^{-1} F$
\State Construct the fairness-adjusted matrix $U_{\text{sym}} = G_{\mathrm{sym}} - M (F^T M)^{-1} F^T G_{\mathrm{sym}}$
\State Compute the top $k$ eigenvectors of $U$ corresponding to largest eigenvalues; let them form $H \in \mathbb{R}^{n \times k}$
\State Apply $k$-means clustering to the rows of $H$
\end{algorithmic}
\end{algorithm}

Based on Algorithm \eqref{alg:smw-cluster}, we can construct two alternative algorithms by substituting $G_{\text{sym}}$ with either $G_{\text{rw}} = D^{-1} W + 2 I$ or $G_{\text{aff}} = W+nI$.

\section{Time Complexity} 

In this section, we compute the asymptotic runtime under the assumption that the weighted adjacency matrix $W$, the degree matrix $D$, and the group-membership matrix $F$ are stored in sparse format in MATLAB. Fair-SMW shares a similar problem to S-Fair-SC, as both algorithms' complexity is dominated by computing $k$ eigenvectors in MATLAB. For each Krylov subspace reset, the complexity depends on the cost of matrix-vector products with $U$ within the function handle, and the computation of the Krylov subspace using MATLAB’s `eigs` function. Thus, using a modern Krylov subspace eigensolver such as ARPACK (via MATLAB’s eigs), the per-restart cost is $O(m + n h^2 + h^3 + n k^2)$, where $k$ is the number of eigenpairs, $h$ is the number of groups, and $n$ is the number of vertices \cite{lehoucq1998implicit}.

%\textcolor{red}{The Pre and After Computations of $k$ eignvalues has the same asympotatic $O(N)$ time complexity than s-FairSC. Is there a way to prove that our method uses fewer restarts? (IK, we saw once that it did, but why???)}

%%%%%%%%%%%%%%%%%%%%%%%%%%%%%%%%%%%%%%%%%%%%%
% SECTION: Experimental Results
% Comments:
% TODO:
%   1.
%%%%%%%%%%%%%%%%%%%%%%%%%%%%%%%%%%%%%%%%%%%%
\section{Experiments}
In this section, we evaluated the performance of the three variations of the Fair-SMW Model in comparison to the normalized-SC, S-Fair-SC, and Fair-SC algorithms. In MATLAB simulations, we analyzed average balance, overall runtime, and eigensolver-specific computation time on four publicly available datasets and one synthetic dataset. As well as perfromed a small test to see how our algorithms perform on dense and sparse matrix. To ensure a fair comparison, we modified the S-Fair-SC and Normalized-SC algorithms from \footnote{\label{fn: WangCode}\url{https://github.com/jiiwang/scalable_fair_spectral_clustering}} to enforce sparsity on the affinity matrix. The results were yielded using an MSI Modern 14 C7M with an AMD Ryzen 5 7530U processor (2.00 GHz) with Radeon Graphics, 16 GB of RAM, and 16 MB of L3 cache. The modified code is available in \href{https://github.com/lcambisaca/Fair-SMW}{github} \footnote{\url{https://github.com/lcambisaca/Fair-SMW}}.

\subsection{Data Sets}

\textbf{FacebookNET} The FacebookNet dataset is based on the High-School Contact and Friendship Networks gathered in December 2013 from a French high school (SocioPatterns). It comprises \(n = 155\) students represented as vertices, with edges indicating reported Facebook friendships. The students are categorized by gender into two groups: \(V_1\), which includes 70 girls, and \(V_2\), consisting of 85 boys. \\
\textbf{LastFM} The LastFM dataset is a social network dataset collected from Last.FM music platform, available through the Stanford Network Analysis Project (Rozemberczki and Sarkar, 2020). It contains $n = 5{,}576$ users from Asian countries, each represented as a vertex in a graph where edges indicate mutual follower relationships. The network comprises $m = 19{,}577$ undirected edges with a density of approximately $1.3 \times 10^{-4}$. \\
\textbf{Deezer} The Deezer Europe Social Network dataset is a user–item social graph collected from the music streaming platform Deezer using its public API in March 2020 (Rozemberczki and Sarkar, 2020). The network consists of $n = 28{,}281$ users from multiple European countries, where each node represents a user and undirected edges reflect mutual follower relationships. The graph includes $m = 92{,}752$ edges, with a density of approximately $2.0 \times 10^{-4}$ and a transitivity (clustering coefficient) of about $0.096$. Each node is enriched with feature vectors derived from the artists the user likes, and users are labeled with binary class information (gender), inferred from their name fields. \\
\textbf{German} The German Credit dataset is a tabular dataset originally derived from the Statlog project and is available through the UCI Machine Learning Repository (Hofmann, 1994). The original dataset comprises n = 1000 loan applicants, each labeled as either a “good” or “bad” credit risk, and is widely used for credit risk modeling, discrimination detection, and fairness-aware learning (I need to cite a paper from Dr. Dong's paper). Each applicant is described by 27 attributes, comprising a mix of numerical and categorical variables. In the version used in our experiments, we adopt the fairness-augmented, preprocessed variant from the Graph Mining Fairness Data repository (Dong et al., 2022), which expands the feature set to 27 columns. All our cleaned data sets where gathered from \cite{wang2023}\textsuperscript{\ref{fn: WangCode}} \cite{Li_2023} \footnote{\url{https://github.com/JiaLi2000/FNM}} Raw Datasets link can be found below \footnote{
\textbf{LastFM:}
\href{http://snap.stanford.edu/data/feather-lastfm-social.html}{http://snap.stanford.edu/data/feather-lastfm-social.html} \\
\textbf{FacebookNet:}\href{http://www.sociopatterns.org/datasets/high-school-contact-and-friendship-networks/}{ http://www.sociopatterns.org/datasets/high-school-contact-and-friendship-networks/} \\
\textbf{Deezer:} \href{https://snap.stanford.edu/data/feather-deezer-social.html}{https://snap.stanford.edu/data/feather-deezer-social.html} \\
\textbf{German:}\href{https://github.com/yushundong/Graph-Mining-Fairness-Data/tree/main/dataset/german}{https://github.com/yushundong/Graph-Mining-Fairness-Data/tree/main/dataset/\\german}
}
%\textit{Friendship}, \textit{LastFM}, and \textit{Deezer} are all social network datasets, while \textit{German} is a similarity graph created from i.i.d. tabular data. \textit{SBM} is a similarity matrix generated from a stochastic block model with random group assignments. \cite{Li_2023}. 

% Note need to change figure names
\section{Results and Analysis}

\textbf{Experiment 1:} 
We used the generate SBM function from \cite{wang2023} as a baseline test to examine algorithms scalability. In this experiment, we evaluated the average balance, total runtime, and eigs runtime across SC, S-Fair-SC, AFF-Fair-SMW, SYM-Fair-SMW, and RW-Fair-SMW. Figure \ref{fg - SBMBalance} reports computational time across model sizes ranging from $n=1000$ to $n=10000$, with $h=2$, $k=2$, and edge probability proportional to $\left(\tfrac{\log n}{n}\right)^{2/3}$. Similar to S-Fair-SC, our three graph variants recovered the fair ground truth clustering with identical error rates. Figure \ref{fg - SBMRunTime} compares total runtime, showing that AFF-Fair-SMW and RW-Fair-SMW achieved faster performance than S-Fair-SC by reducing computational overhead outside the \texttt{eigs} routine.  Figure \ref{fg - EigsSBMRunTime} demonstrates that the \texttt{eigs} component itself contributed only a small fraction of the overall runtime. \\

\textbf{Experiment 2:} We used the FacebookNet dataset to evaluate all algorithms across $k = 2$ to $k = 15$. Following \cite{chierichetti2017}, fairness in clustering was measured using the average balance, defined for a cluster $\mathcal{C}_\ell$ as in equation \ref{balance}.
The average balance is then given by
\[
\text{Average\_Balance} := \frac{1}{k} \sum_{\ell=1}^{k} \text{balance}(\bm{C}_\ell).
\]

Figure \ref{fg - FriendshipBalance} uses the balance definition from \cite{wang2023},  illustrates the average balance (as defined in \cite{wang2023}), where higher values indicate fairer clusterings. 
We observe that AFF-Fair-SMW consistently maintains an average balance above 50\%, while the other algorithms achieve comparable fairness levels. Figures \ref{fg - FriendshipRunTime} and \ref{fg - EigsFriendshipRunTime} show that all methods have similar runtimes, with the majority of computation occurring outside the \texttt{eigs} routine.

\textbf{Experiment 3:} 
We used the LastFMNet dataset to evaluate average balance and runtime for all algorithms across $k = 2$ to $15$. As shown in Figure \ref{fg - LastfmBalance}, our algorithms maintained an average balance comparable to S-Fair-SC. Figure \ref{fg - LastfmRunTime} presents total runtimes, where for $k = 2$ to $4$, AFF-Fair-SMW significantly outperformed S-Fair-SC. Figure \ref{fg - EigsLastfmRunTime} explains this improvement, showing that it is largely due to reduced runtime in the \texttt{eigs} solver. These results highlight that when the \texttt{eigs} solver dominates time complexity, AFF-Fair-SMW is the most efficient choice.

\textbf{Experiment 4:} 
We used the German dataset to evaluate average balance and runtime for all algorithms across $k = 2$ to $15$. Figure \ref{fg - GermanBalance} shows that our methods maintained an average balance comparable to S-Fair-SC. However, Figure \ref{fg - GermanRunTime} reveals that AFF-Fair-SMW incurred the highest overall runtime, while the remaining algorithms exhibited similar performance. Figure \ref{fg - EigsGermanRunTime} indicates that this increased complexity is not due to the \texttt{eigs} solver. 
Additional testing suggests that the overhead arises from the $k$-means step, since our overall algorithm performs fewer matrix operations. This highlights the importance of considering the efficiency of the $k$-means component when evaluating total runtime.
\\
\textbf{Experiment 5:} 
We used the Deezer dataset to evaluate average balance and runtime for all algorithms across $k = 2$ to $15$. Figure \ref{fg - DeezerBalance} shows that all algorithms exhibit fluctuations in balance, sometimes outperforming and other times underperforming S-Fair-SC, yet maintaining a relatively high average balance overall. Figure \ref{fg - DeezerRunTime} highlights runtime comparisons, where AFF-Fair-SMW achieves a dramatic reduction in runtime, lowering the initial clustering time from over 30 seconds to under one second. Figure \ref{fg - EigsDeezerRunTime} illustrates that this improvement stems from reduced \texttt{eigs} solver complexity, reinforcing that when the \texttt{eigs} function dominates total runtime, AFF-Fair-SMW is the most efficient algorithm.

\textbf{Additional Experiment: Dense vs. Sparse Matrices.} 
We further examined the impact of matrix density on eigs runtime by testing both dense and sparse graphs. For the sparse case, we used a stochastic block model (SBM) and modified the algorithm to construct a spanning tree, ensuring the graph contained at least one strongly connected component suitable for spectral clustering. Our results revealed several trends. For very dense matrices, the overall runtime increased significantly as $n$ grew, with all algorithms performing similarly to S-Fair-SMW. In this setting, AFF-Fair-SC showed only modest improvements, and the observed runtime differences were largely due to the \texttt{eigs} solver. 

In contrast, for very sparse matrices we observed dramatic improvements as AFF-Fair-SMW reduced runtime to around $0.2$ seconds, while S-Fair-SC required approximately $3.5$ seconds. As $n$ increased, the runtime of S-Fair-SC and related methods grew substantially, whereas AFF-Fair-SMW maintained stable performance. Although the exact reason for this robustness remains uncertain, we suspect that the eigen-gap plays a critical role; preliminary analysis suggests a large eigen-gap in the sparse case, but further research is needed to confirm this.  

To further evaluate the robustness of our algorithms, we simulated a checkerboard graph—a hypothetically challenging graph that often poses difficulties for spectral clustering. Our results show that AFF-Fair-SMW successfully converges on this structure. Additionally, we tested extremely sparse matrices, where AFF-Fair-SMW reliably converged to accurate eigenvalues, while S-Fair-SC, Sym-Fair-SMW, and RW-Fair-SMW occasionally failed to converge due to limitations in their eigen-solvers. Although additional iterations could mitigate these failures, our findings suggest that AFF-Fair-SMW is uniquely resilient, effectively converging under both sparse and dense conditions.

\label{sec:experiments} %NEED TO CHNAGE FIG LABELS SADDDDDDDDD
\begin{figure}[htp]
    \centering
    \begin{subfigure}[b]{0.28\textwidth}
\includegraphics[width=\textwidth]{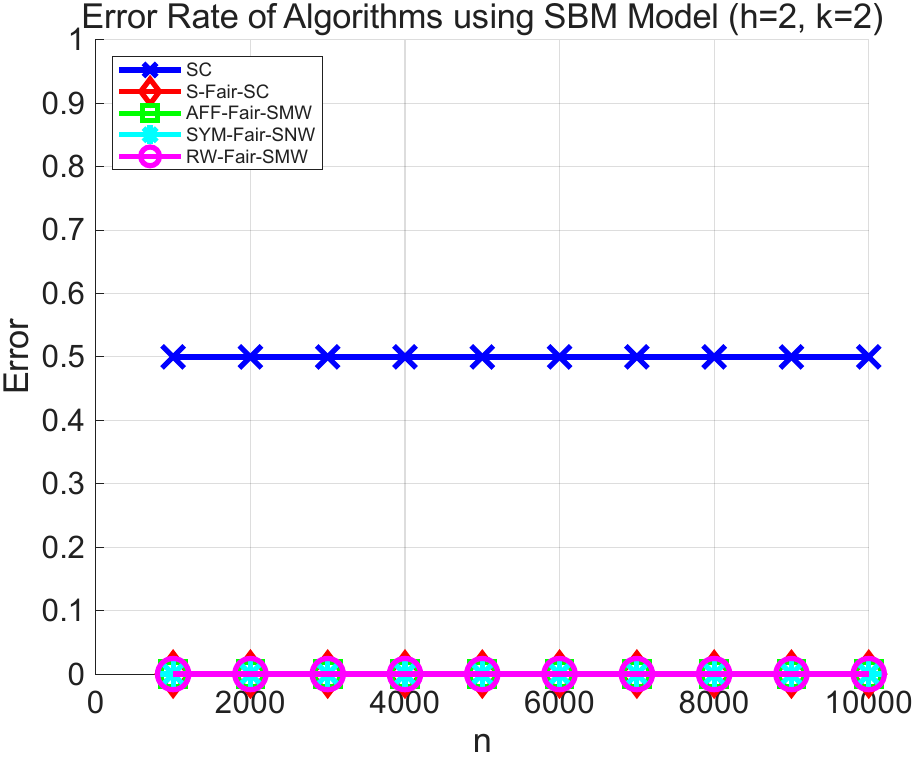}
        \caption{SBM - Error}
        \label{fg - SBMBalance}
    \end{subfigure}
    \hfill
    \begin{subfigure}[b]{0.28\textwidth}
        \includegraphics[width=\textwidth]{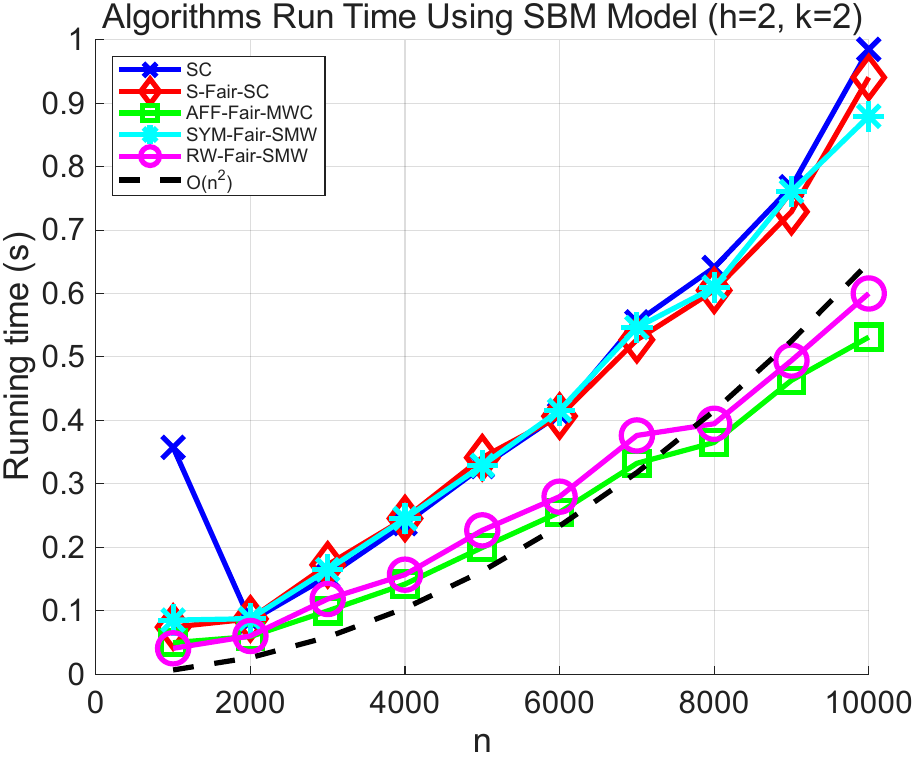}
        \caption{SBM - Run Time}
        \label{fg - SBMRunTime}
    \end{subfigure}
    \hfill
    \begin{subfigure}[b]{0.28\textwidth}
        \includegraphics[width=\textwidth]{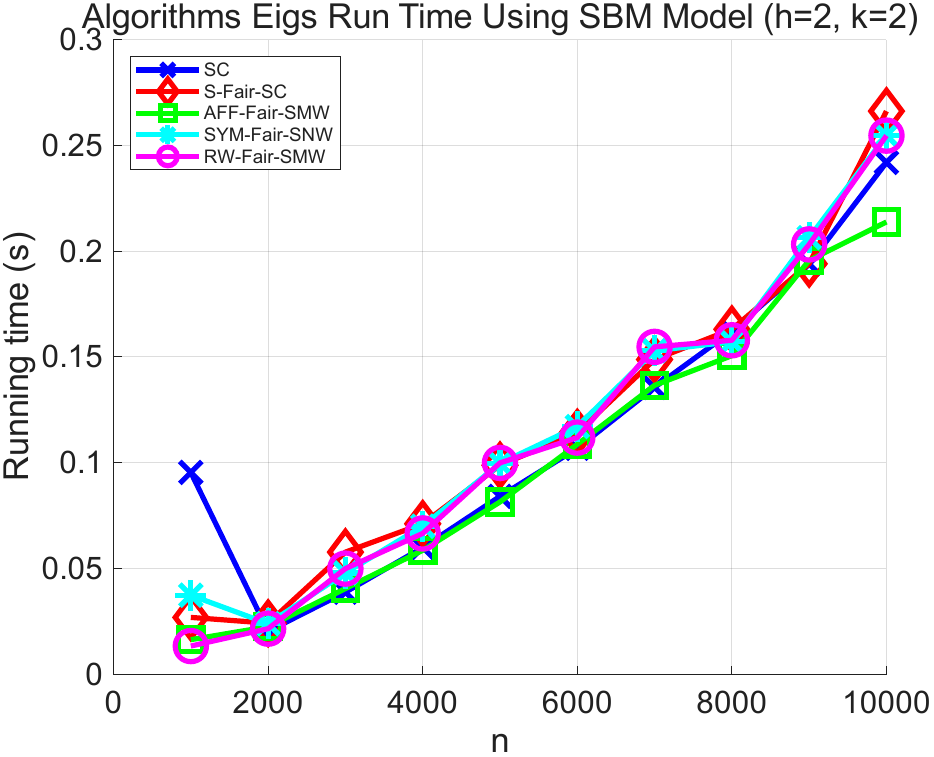}
        \caption{SBM - Eigs Time}
        \label{fg - EigsSBMRunTime}
    \end{subfigure}

    %\vspace{-0.7em}

    % Row 2: Friendship
    \begin{subfigure}[b]{0.28\textwidth}
        \includegraphics[width=\textwidth]{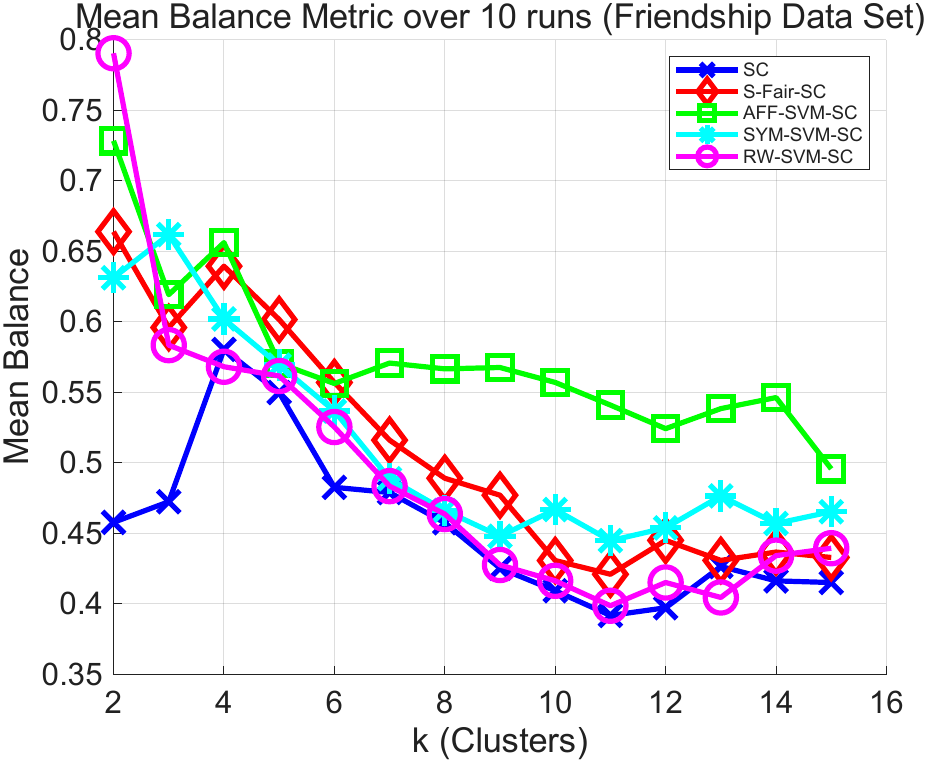}
        \caption{Friendship - Balance}
        \label{fg - FriendshipBalance}

    \end{subfigure}
    \hfill
    \begin{subfigure}[b]{0.28\textwidth}
        \includegraphics[width=\textwidth]{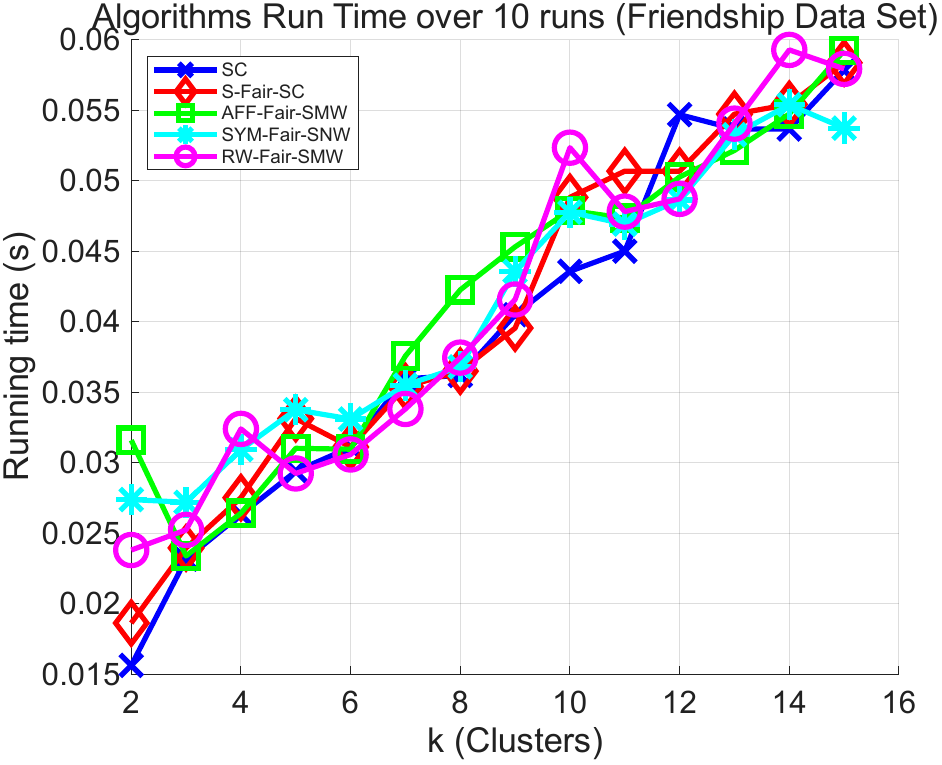}
        \caption{Friendship - Run Time}
        \label{fg - FriendshipRunTime}
    \end{subfigure}
    \hfill
    \begin{subfigure}[b]{0.28\textwidth}
        \includegraphics[width=\textwidth]{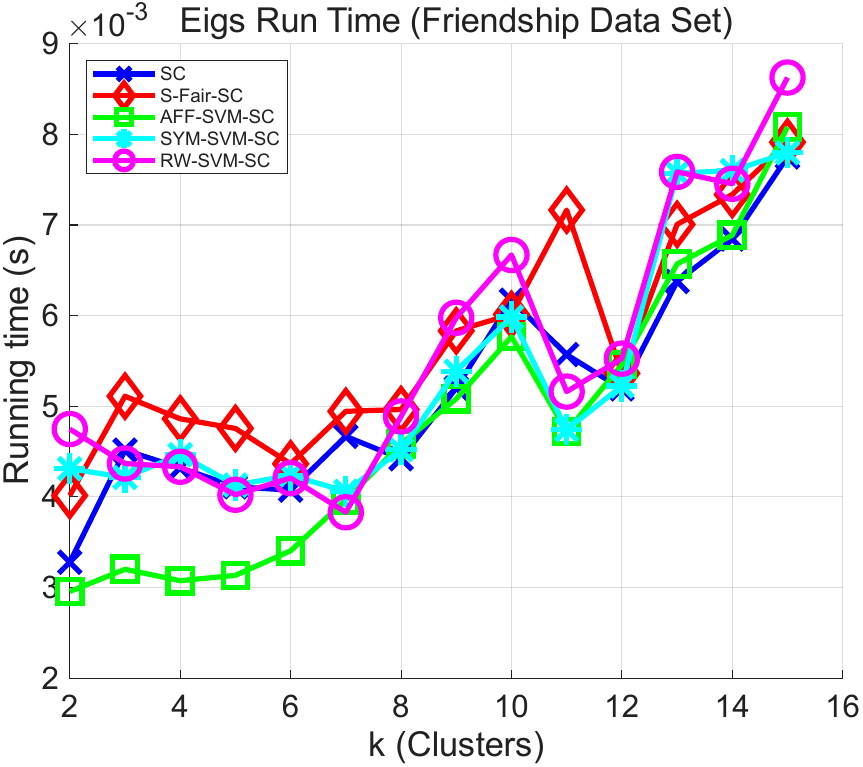}
        \caption{Friendship - Eigs Time}
        \label{fg - EigsFriendshipRunTime}

    \end{subfigure}

    %\vspace{-0.7em}

    % Row 3: LastFm
    \begin{subfigure}[b]{0.28\textwidth}
        \includegraphics[width=\textwidth]{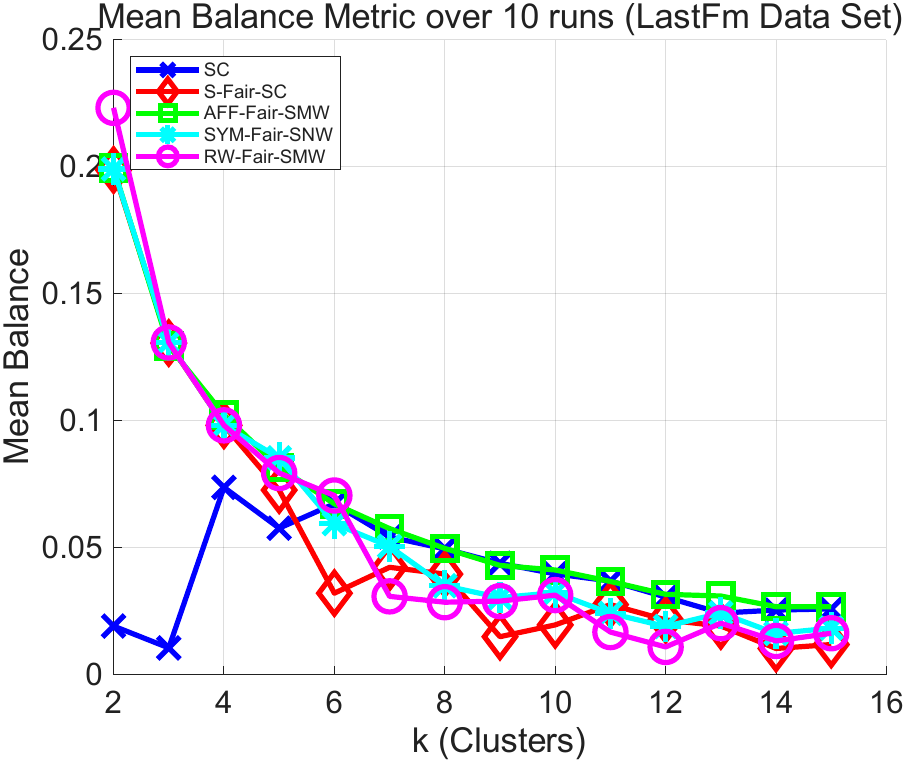}
        \caption{LastFm - Balance}
        \label{fg - LastfmBalance}

    \end{subfigure}
    \hfill
    \begin{subfigure}[b]{0.28\textwidth}
        \includegraphics[width=\textwidth]{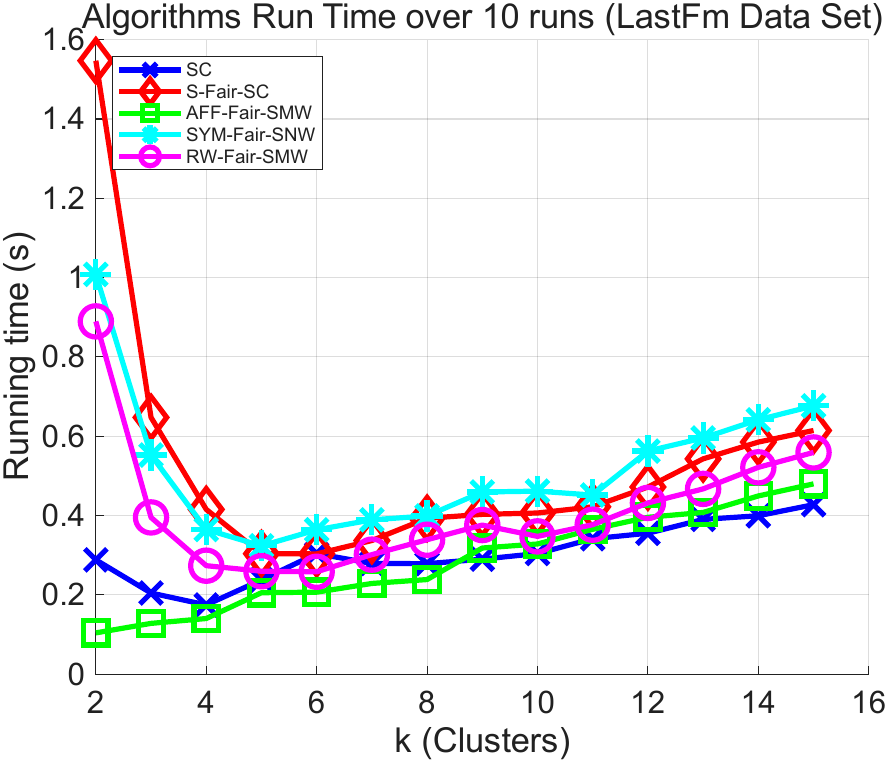}
        \caption{LastFm - Run Time}
        \label{fg - LastfmRunTime}

    \end{subfigure}
    \hfill
    \begin{subfigure}[b]{0.28\textwidth}
        \includegraphics[width=\textwidth]{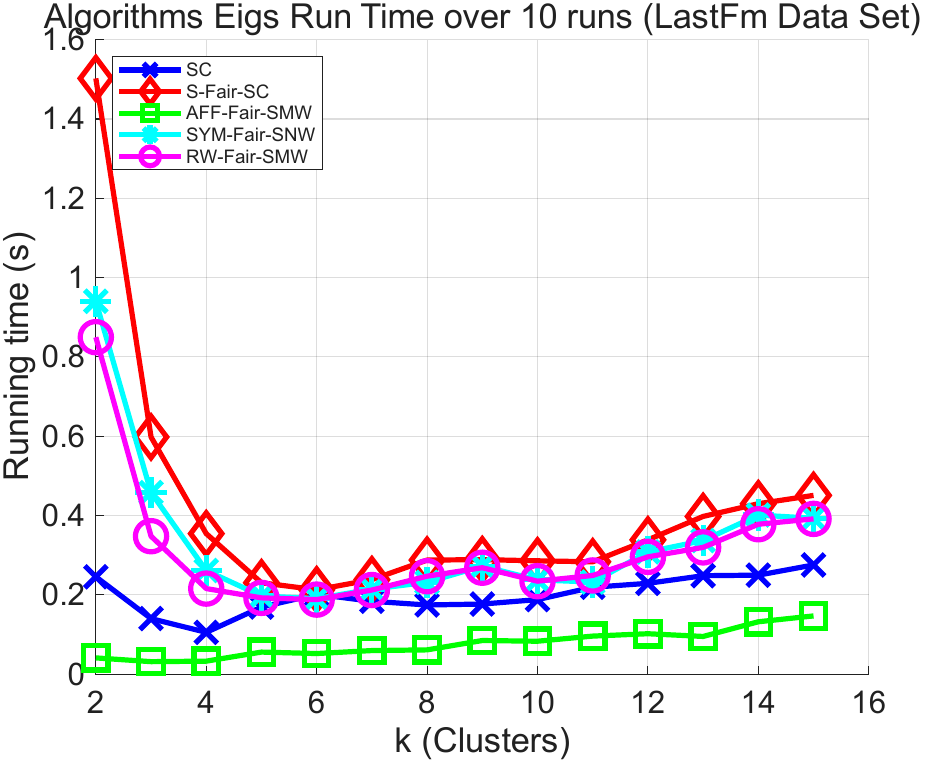}
        \caption{LastFm - Eigs Time}
        \label{fg - EigsLastfmRunTime}

    \end{subfigure}

    %\vspace{-0.7em}

    % Row 4: German
    \begin{subfigure}[b]{0.28\textwidth}
        \includegraphics[width=\textwidth]{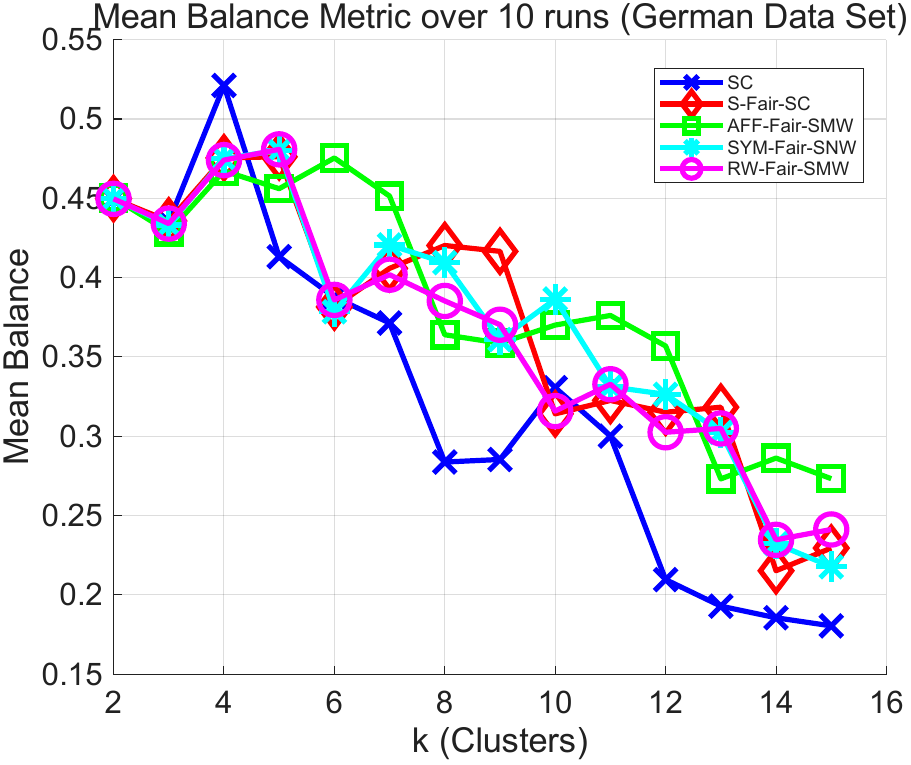}
        \caption{German - Balance}
        \label{fg - GermanBalance}

    \end{subfigure}
    \hfill
    \begin{subfigure}[b]{0.28\textwidth}
        \includegraphics[width=\textwidth]{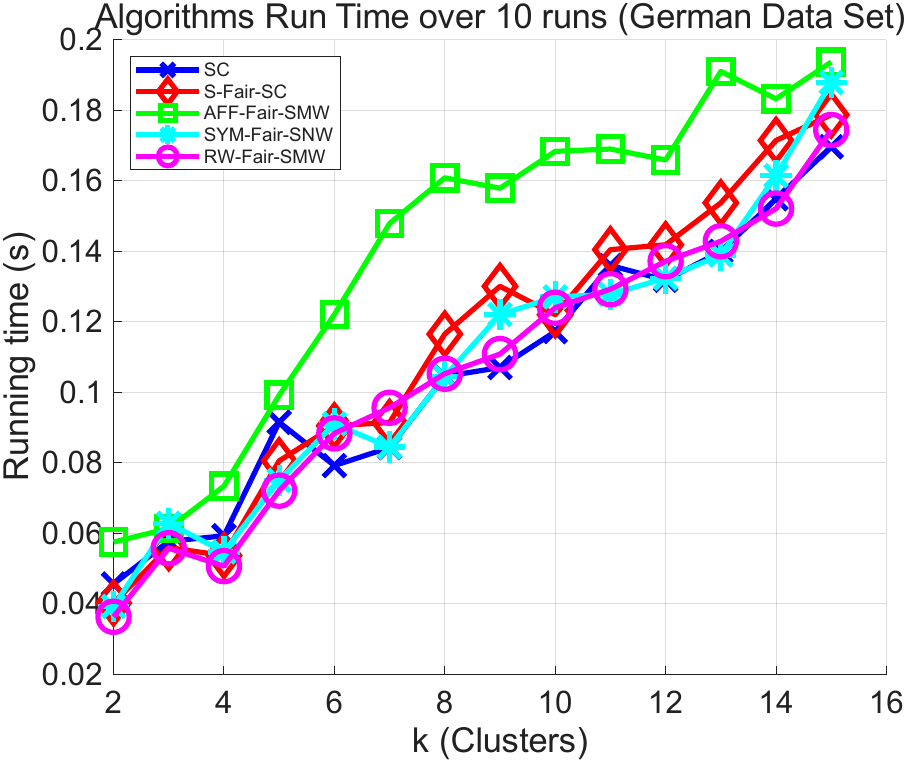}
        \caption{German - Run Time}
        \label{fg - GermanRunTime}

    \end{subfigure}
    \hfill
    \begin{subfigure}[b]{0.28\textwidth}
        \includegraphics[width=\textwidth]{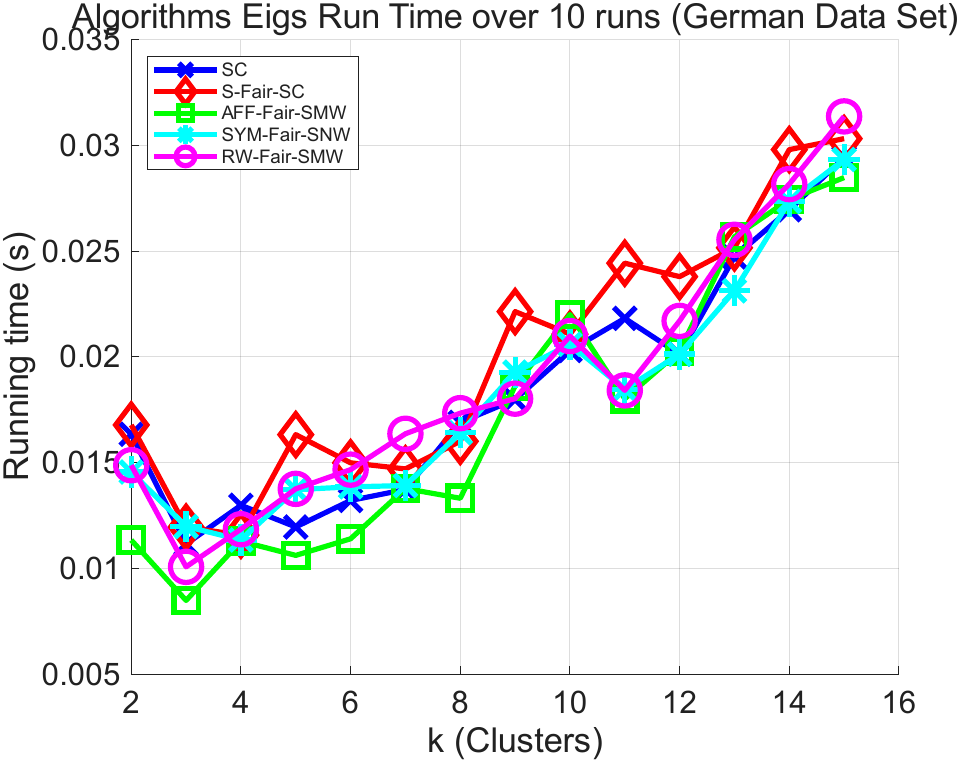}
        \caption{German - Eigs Time}
        \label{fg - EigsGermanRunTime}

    \end{subfigure}

    %\vspace{-0.7em}

    % Row 5: Deezer
    \begin{subfigure}[b]{0.28\textwidth}
        \includegraphics[width=\textwidth]{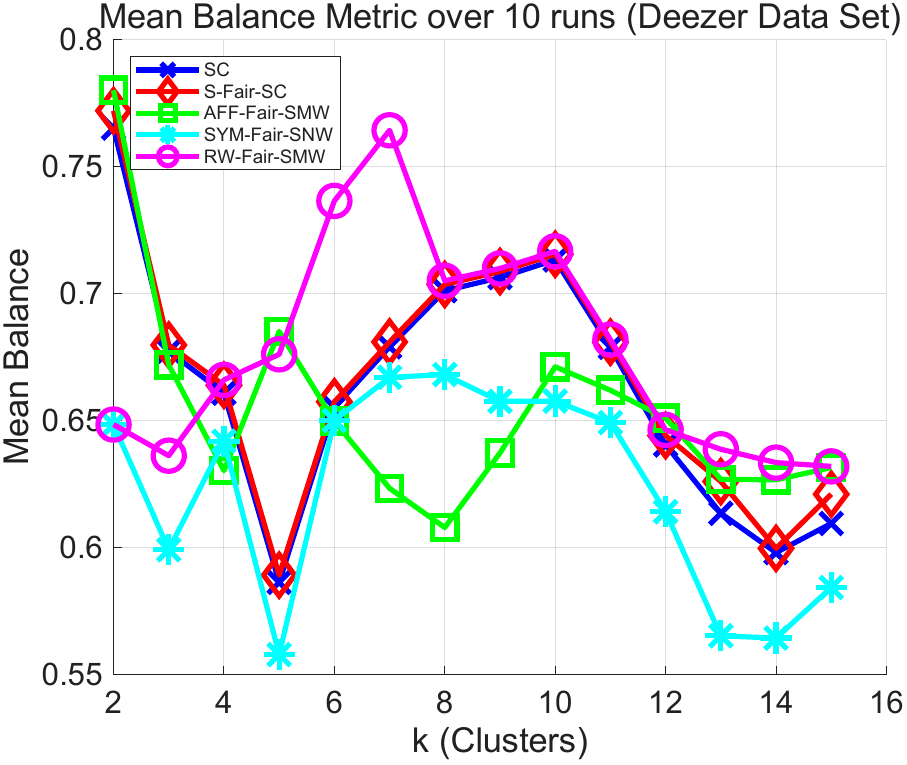}
        \caption{Deezer - Balance}
        \label{fg - DeezerBalance}

    \end{subfigure}
    \hfill
    \begin{subfigure}[b]{0.28\textwidth}
        \includegraphics[width=\textwidth]{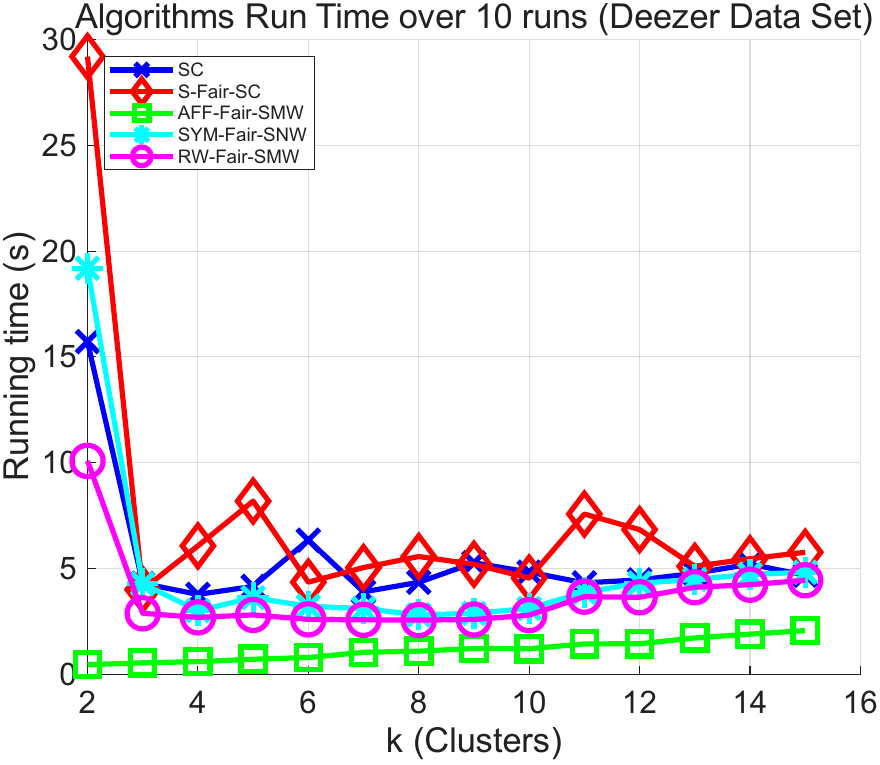}
        \caption{Deezer - Run Time}
        \label{fg - DeezerRunTime}

    \end{subfigure}
    \hfill
    \begin{subfigure}[b]{0.28\textwidth}
        \includegraphics[width=\textwidth]{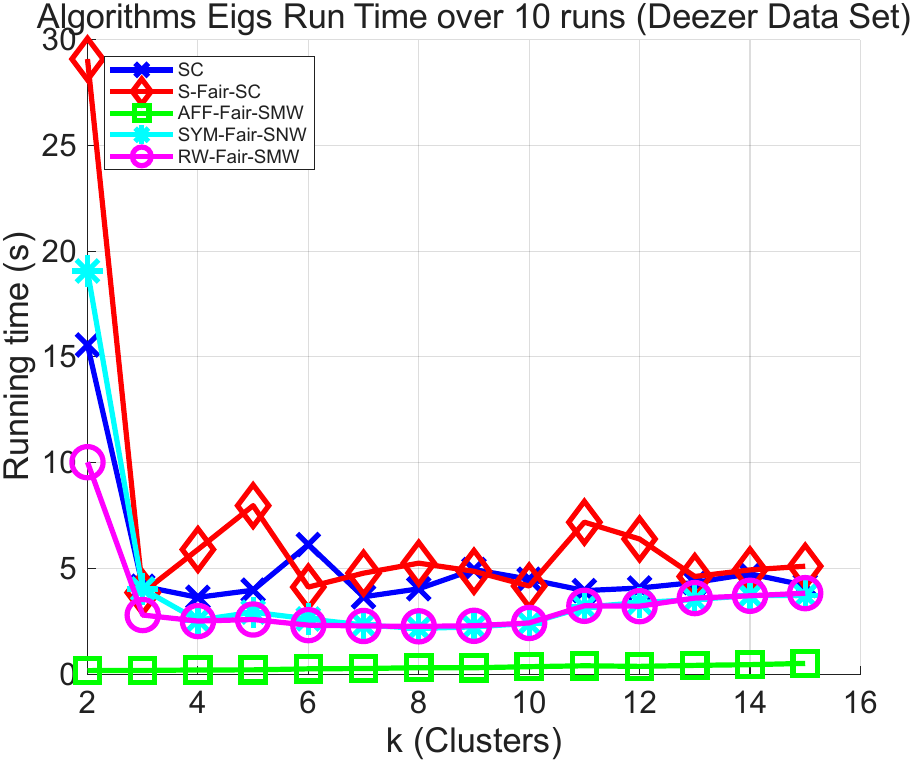}
        \caption{Deezer - Eigs Time}
        \label{fg - EigsDeezerRunTime}

    \end{subfigure}

    %\vspace{-1em}
    \caption{(Left) Balance, (Center) Algorithm Run Time, (Right) Eigensolver Run Time}
\end{figure}
\clearpage  % Force previous content to flush

%%%%%%%%%%%%%%%%%%%%%%%%%%%%%%%%%%%%%%%%%%%%%
% SECTION: Conclusions
% Comments:
% TODO:
%   1.
%%%%%%%%%%%%%%%%%%%%%%%%%%%%%%%%%%%%%%%%%%%%%

\begin{figure}[ht]
    \centering
    % Row 1
    \begin{subfigure}[b]{0.35\textwidth}
        \includegraphics[width=\linewidth]{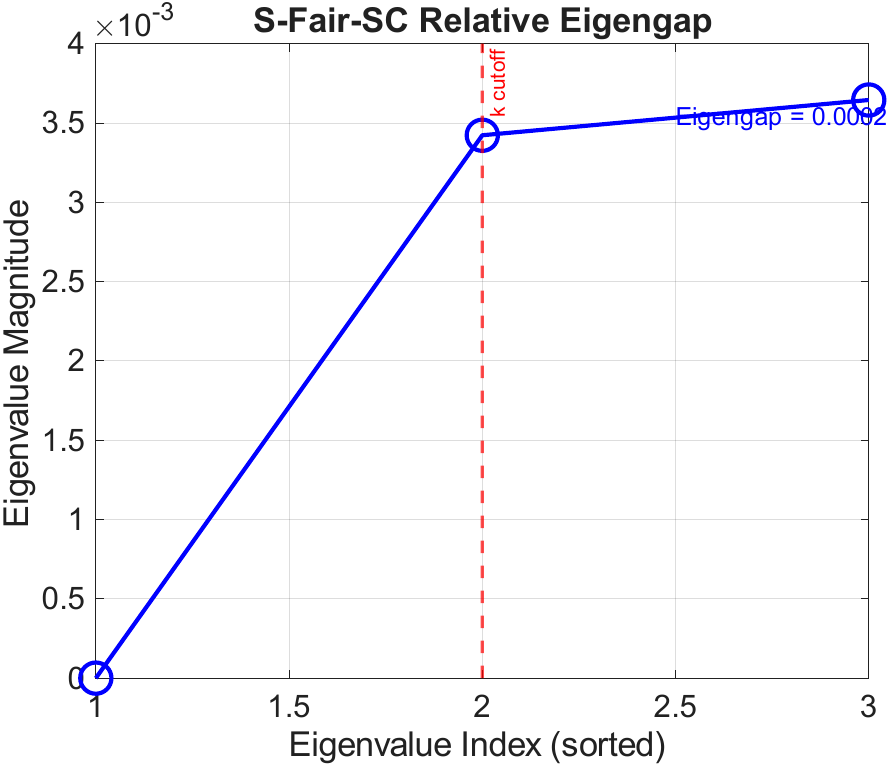}
        \caption{Deezer Eigen Gap}
    \end{subfigure}
    \hfill
    \begin{subfigure}[b]{0.35\textwidth}
        \includegraphics[width=\linewidth]{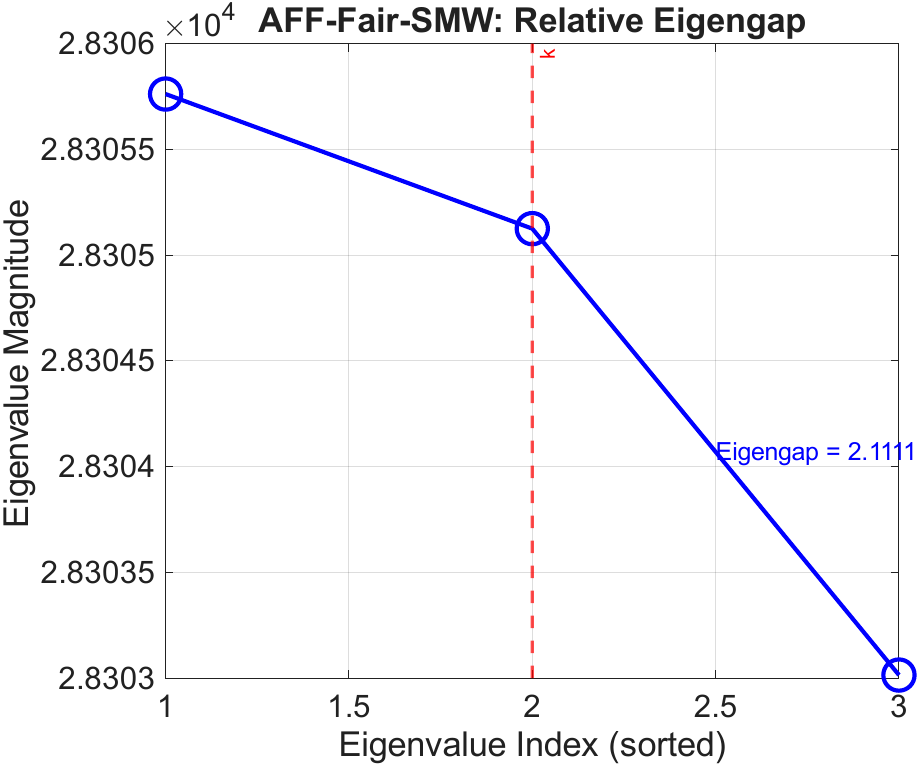}
        \caption{Deezer Eigen Gap}
    \end{subfigure}

    \medskip

    % Row 2
    \begin{subfigure}[b]{0.35\textwidth}
        \includegraphics[width=\linewidth]{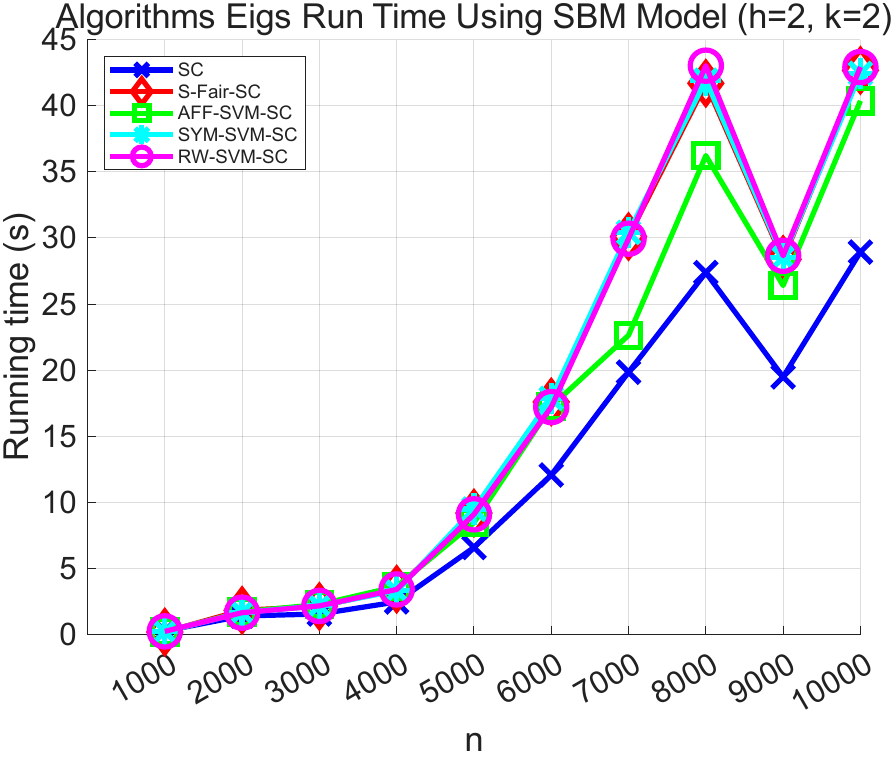}
        \caption{Dense Eigs RunTime}
    \end{subfigure}
    \hfill
    \begin{subfigure}[b]{0.35\textwidth}
        \includegraphics[width=\linewidth]{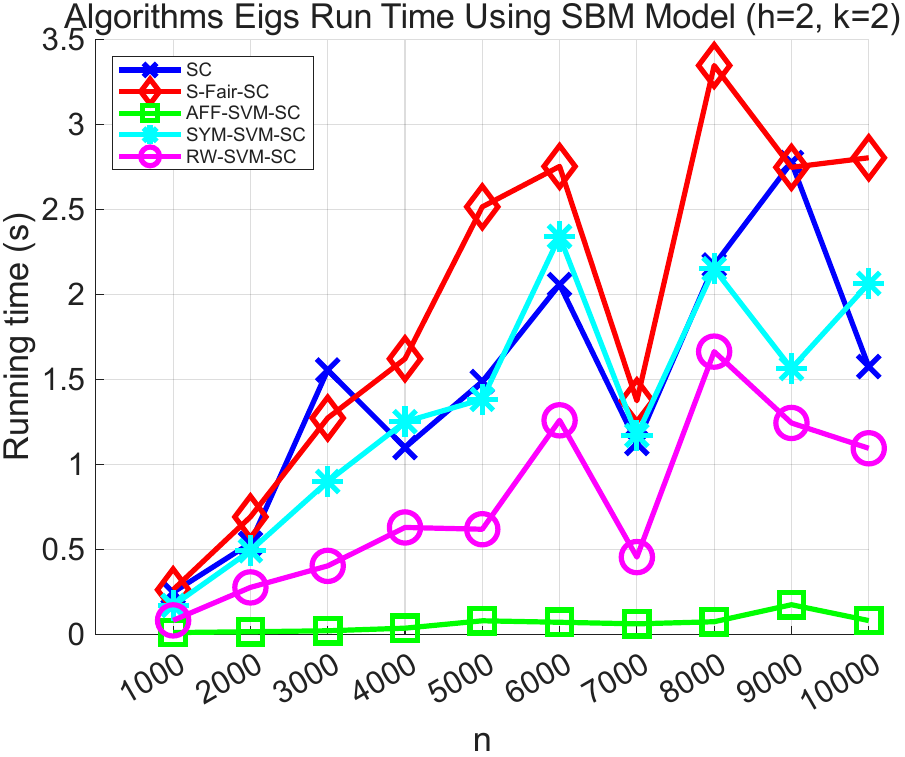}
        \caption{Sparse Eigs RunTime}
    \end{subfigure}

    \medskip

    % Row 3
    \begin{subfigure}[b]{0.32\textwidth}
        \includegraphics[width=\linewidth]{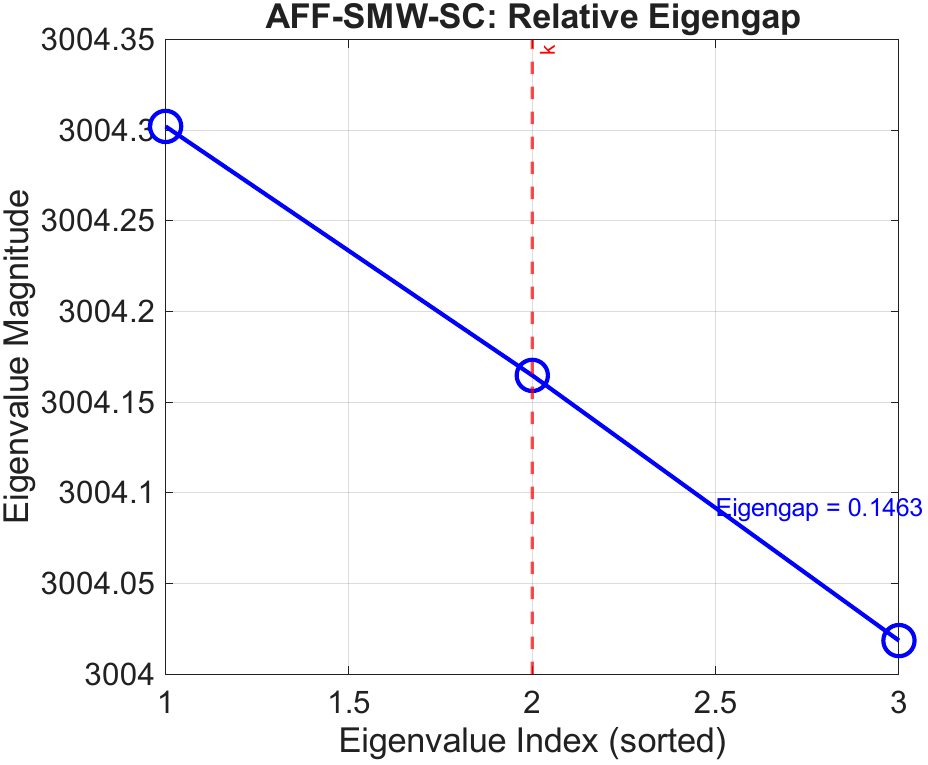}
        \caption{$n=3000$}
    \end{subfigure}
    \hfill
    \begin{subfigure}[b]{0.32\textwidth}
        \includegraphics[width=\linewidth]{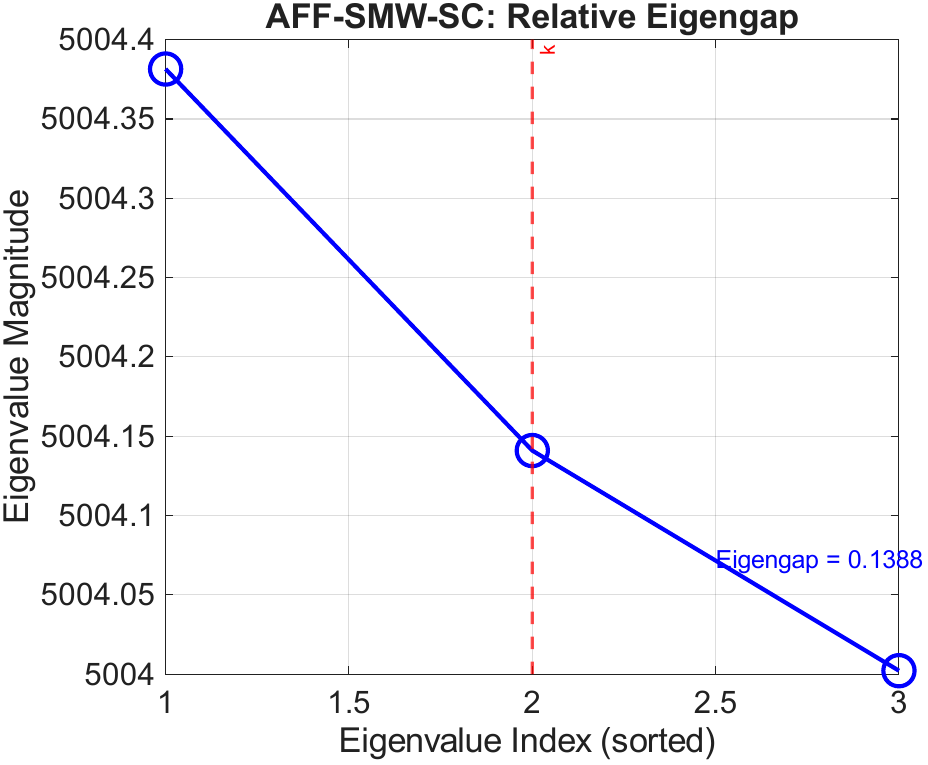}
        \caption{$n=5000$}
    \end{subfigure}
    \hfill
    \begin{subfigure}[b]{0.32\textwidth}
        \includegraphics[width=\linewidth]{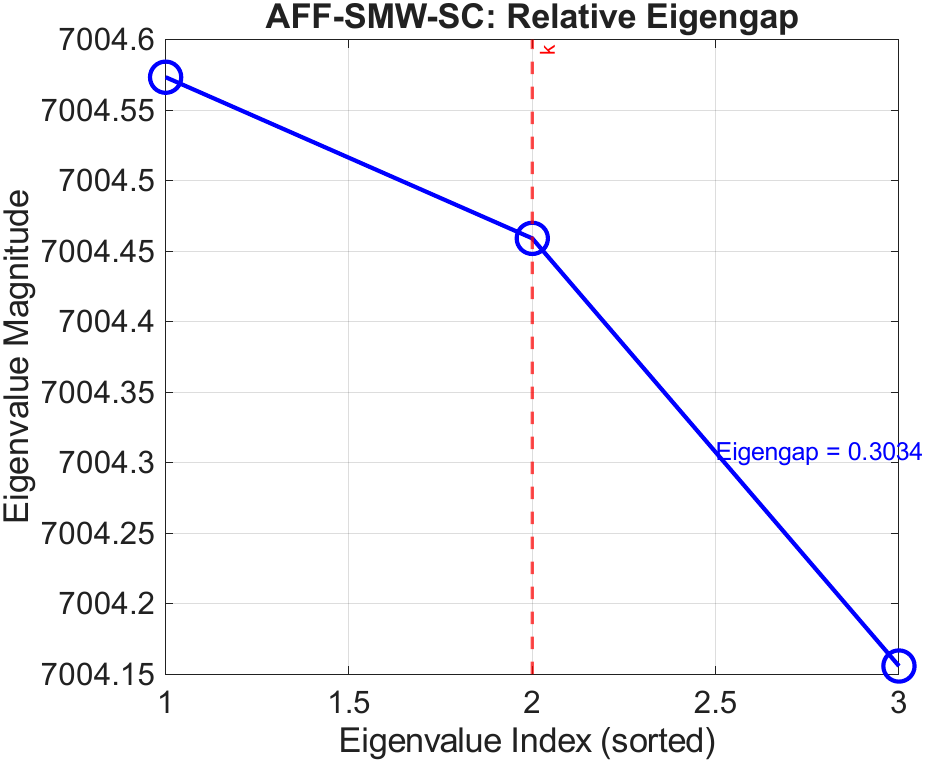}
        \caption{$n=7000$}
    \end{subfigure}

    \medskip

    % Row 4
    \begin{subfigure}[b]{0.32\textwidth}
        \includegraphics[width=\linewidth]{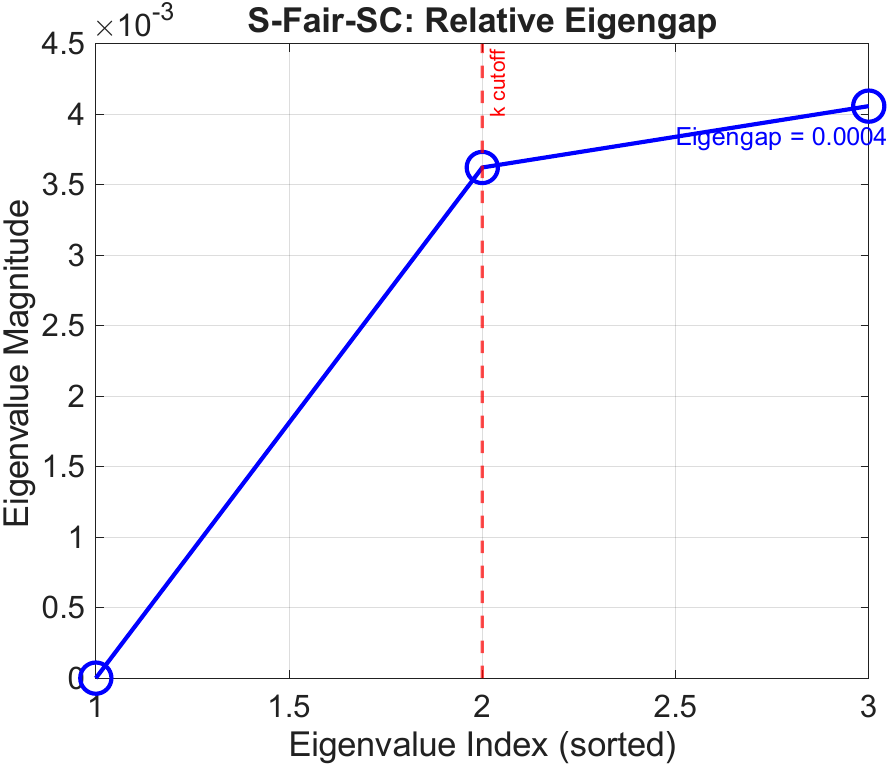}
        \caption{$n=3000$}
    \end{subfigure}
    \hfill
    \begin{subfigure}[b]{0.32\textwidth}
        \includegraphics[width=\linewidth]{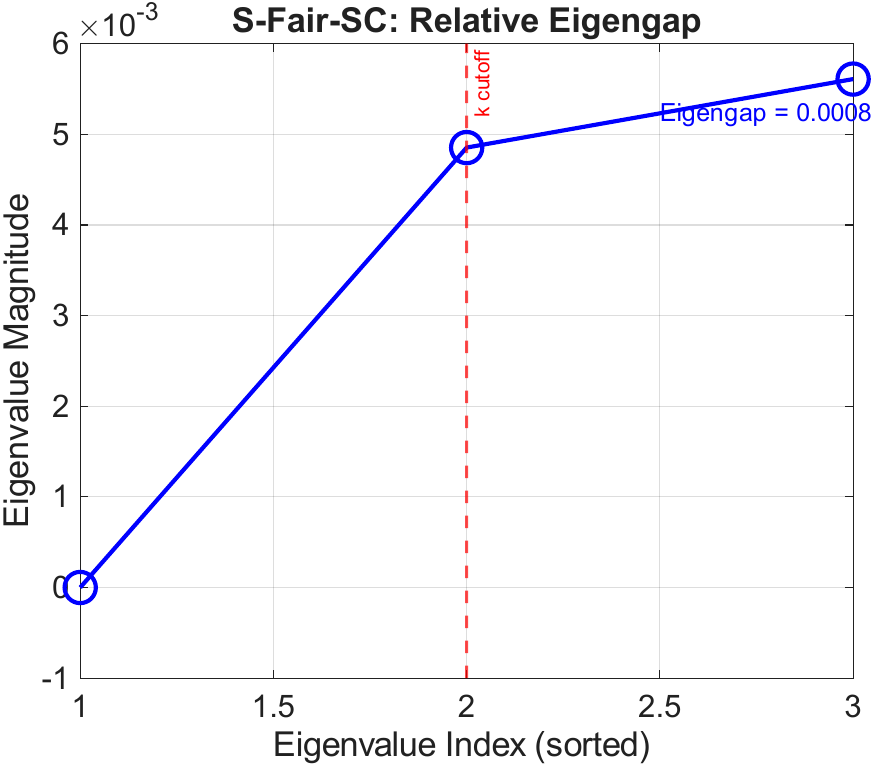}
        \caption{$n=5000$}
    \end{subfigure}
    \hfill
    \begin{subfigure}[b]{0.32\textwidth}
        \includegraphics[width=\linewidth]{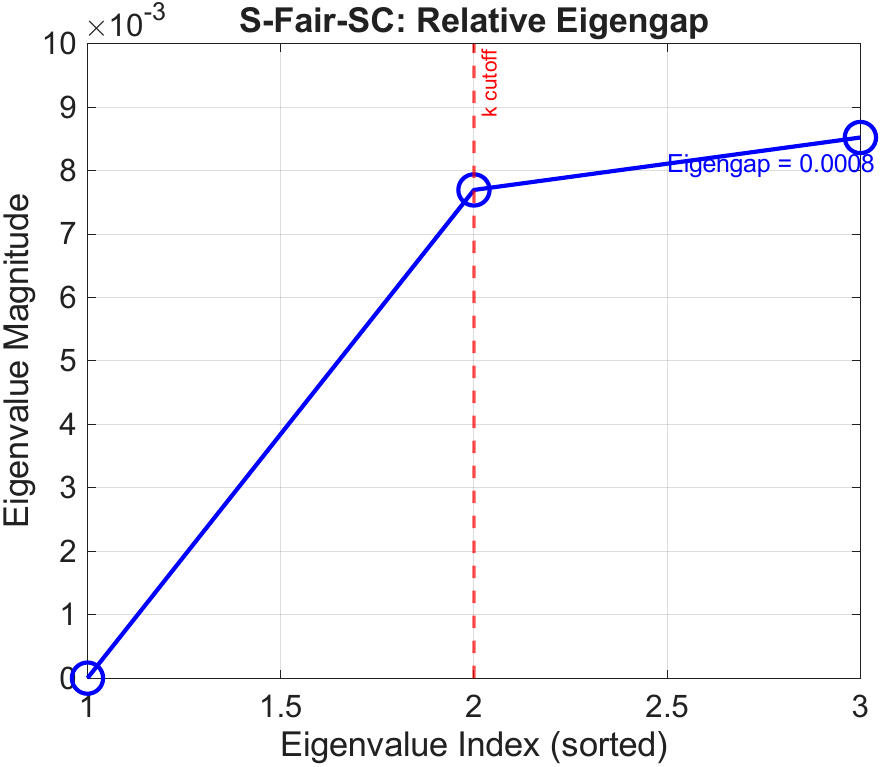}
        \caption{$n=7000$}
    \end{subfigure}

    \caption{Comparison of eigenvalue gaps, runtimes, and scalability for SMW and S-Fair-SC methods across different datasets and graph sizes.}
\end{figure}

\section{Conclusion}
Across all experiments, AFF-Fair-SMW achieved fairness levels comparable to existing baselines while offering substantial runtime advantages. On dense graphs, improvements over S-Fair-SC were modest due to the higher computational cost of large affinity matrices. On sparse graphs, the benefits were dramatic: S-Fair-SC required hundreds of eigensolver restarts (e.g., 605 on Deezer), whereas AFF-Fair-SMW converged in only a few iterations (e.g., 14). This reduction in iterations translated into a massive improvement of algorithm runtime. While spectral properties such as eigen-gap may contribute to this behavior, the most consistent and reproducible finding is that AFF-Fair-SMW substantially reduces eigensolver iterations. These results highlight the robustness of AFF-Fair-SMW, which maintains competitive balance, scales efficiently to larger cluster sizes, and delivers especially strong performance when eigs solver dominates run time.

We propose AFF-Fair-SMW, a fair spectral clustering method that consistently maintains competitive balance while achieving substantially faster eigensolver convergence compared to existing baselines. 
Empirically, it demonstrates robust performance across sparse and dense graphs, dramatically reducing solver iterations and runtime in sparse regimes. These findings suggest that AFF-Fair-SMW is a practical and reliable method for fair clustering on large-scale graphs and provides a foundation for further exploration of spectral efficiency and scalable clustering techniques.

\appendix
\section{Supporting Proofs} 

\begin{theorem}
    Let $A_{\textup{sym}}F = 0$, then the matrix $G_{\textup{sym}}$ satisfies $$U = [I-F(F^TF)^{-1}F^T]G_{\textup{sym}}.$$
    A similar statement holds for $G_{\textup{rw}}$ and $G_{\textup{aff}}$.
\end{theorem}
\begin{proof}
Let  $A_{\text{sym}}F=0$ and $G_{\text{sym}} = A_{\text{sym}}+2I$ then
    \begin{align*}
    U &= (G_{\text{sym}} + \mu F F^T)^{-1}\\
    &   = G_{\text{sym}} - G_{\text{sym}}F(F^TG_{\text{sym}}F)^{-1}F^TG_{\text{sym}} \\
    &   = [I -G_{\text{sym}}F(F^TG_{\text{sym}}F)^{-1}F^T]G_{\text{sym}} \\
    & =  [I -(A_{\text{sym}} + 2I)F(F^T(A_{\text{sym}} + 2I)F)^{-1}F^T]G_{\text{sym}}\\
    &  =  [I -(A_{\text{sym}}F + 2F )(F^T A_{\text{sym}}F + 2F^TF)^{-1}F^T]G_{\text{sym}}\\
    &   =  [I -2F(2F^TF)^{-1}F^T]G_{\text{sym}} \\
    &   = [I -F(F^TF)^{-1}F^T]G_{\text{sym}}\\
    \end{align*}   
\end{proof}

% \begin{theorem}
    
% \end{theorem}
% \begin{proof}
%  Let \(G = (V,E)\) be an undirected graph with degree \(d_i\) for each vertex \(i\).  
% Recall the normalized Laplacian:
% \[
% L_{\text{sym}} = I - D^{-1/2} W D^{-1/2}.
% \]
% Applying \(L_{\text{sym}}\) to a vector \(\bm{y} \in \mathbb{R}^n\) gives:
% \[
% (L_{\text{sym}}\,\bm{y})(i) = y_i \;-\; \sum_{j \sim i} \frac{y_j}{\sqrt{d_i d_j}}.
% \]
% Observe that we can rewrite the first term as:
% \[
% y_i = \frac{1}{\sqrt{d_i}} \sum_{j \sim i} \frac{y_i}{\sqrt{d_i}},
% \]
% since there are exactly \(d_i\) neighbors \(j\sim i\).
% Therefore:
% \[
% (L_{\text{sym}}\,\bm{y})(i)
% =
% \frac{1}{\sqrt{d_i}} \sum_{j \sim i}\frac{y_i}{\sqrt{d_i}}
% \;-\;
% \sum_{j \sim i}\frac{y_j}{\sqrt{d_i d_j}}.
% \]
% Factoring out \(\frac{1}{\sqrt{d_i}}\), we get:
% \[
% (L_{\text{sym}}\,\bm{y})(i)
% =
% \frac{1}{\sqrt{d_i}}
% \sum_{j \sim i}
% \left(
% \frac{y_i}{\sqrt{d_i}} - \frac{y_j}{\sqrt{d_j}}
% \right).
% \]

% \end{proof}
%%%%
\begin{lemma}\label{lem:diff-square}
For any real numbers \(a,b\), 
\[
(a - b)^2 \;\le\; 2\bigl(a^2 + b^2\bigr).
\]
\end{lemma}

\begin{proof}
We begin with the fact that the square of any real number is non-negative:
\[
0 \le (a + b)^2 = a^2 + 2ab + b^2.
\]

Then:
\begin{align*}
0 &\le a^2 + 2ab + b^2, \\
-2ab &\le a^2 + b^2 & \text{(Subtracting \(2ab\) from both sides)}, \\
a^2 - 2ab + b^2 &\le 2a^2 + 2b^2 & \text{(Adding \(a^2 + b^2\) to both sides)}, \\
(a - b)^2 &\le 2(a^2 + b^2) & \text{(Recognizing the left-hand side as a square)}.
\end{align*}

Hence, the claimed inequality holds for all real \(a, b\).
\end{proof}

% \begin{lemma}\label{lem:sum-of-squares}
% For any function \( f : V \to \mathbb{R} \),
% \[
% \sum_{i \sim j} \left( f(i)^2 + f(j)^2 \right)
% = \sum_{j \in V} d_j\, f(j)^2, \text{ where \( d_j \) is the degree of vertex \( j \).}
% \]
% \end{lemma}

\begin{lemma} \label{diffsqr2}
Let \(G = (V,E)\) be an undirected graph with degree \(d_i\) for each vertex \(i\).  
Recall the Laplacian:
\[L = D-W.\]then
\begin{equation}
\bm{y}^TL \bm{y} =\frac{1}{2}\sum_{i,j} W_{ij}(\bm{y}_i - \bm{y}_j)^2.
\end{equation}
\end{lemma}
\begin{proof}
    \begin{align*}
        \bm{y}^T L \bm{y} &= \bm{y}^T(D-W)\bm{y} = \bm{y}^T D \bm{y} - \bm{y}^T W \bm{y} \\
&= \sum_j \bm{y}_j^2 d_j - \sum_{i=1}^n \sum_{j=1}^n W_{ij} \bm{y}_i \bm{y}_j \\
&= \frac{1}{2}\sum_i \bm{y}_i^2 d_i - \sum_{i,j} W_{ij} \bm{y}_i \bm{y}_j + \frac{1}{2}\sum_j \bm{y}_j^2 d_j \\
&= \frac{1}{2}\sum_{i,j} W_{ij}(\bm{y}_i^2 - 2\bm{y}_i \bm{y}_j + \bm{y}_j^2) = \frac{1}{2}\sum_{i,j} W_{ij}(\bm{y}_i - \bm{y}_j)^2 
    \end{align*}
\end{proof}
It's a fundamental result in spectral graph theory that the eigenvalues of the normalized Laplacian are bounded within the interval [0,2]. While this property is well-established, with a canonical proof provided by Fan Chung in \cite{chung1997spectral}, we include a derivation here to present it from the perspective of our framework and to ensure consistency in notation.
\begin{theorem}
For a weighted, undirected graph, the quadratic form of the unnormalized Laplacian, normalized by the degree matrix, is bounded by 2.
$$ \sup_{\bm{y} \ne \bm{0}} \frac{\bm{y}^T(D-W)\bm{y}}{\bm{y}^T D \bm{y}} = 2 $$
\end{theorem}

\begin{proof}
\begin{align*}
\frac{\bm{x}^T L_{\text{sym}} \bm{x}}{\bm{x}^T \bm{x}} &= \frac{\bm{x}^T D^{-1/2} L D^{-1/2} \bm{x}}{\bm{x}^T \bm{x}} \\
&= \frac{\bm{y}^T L \bm{y}}{\bm{y}^T D \bm{y}} \\
& = \dfrac{\frac{1}{2}\sum_{i,j} W_{ij}(\bm{y}_i - \bm{y}_j)^2}{\sum_j \bm{y}_j^2 d_j}  \quad \text{(By Lemma \ref{diffsqr2})} \\
& \leq \frac{\sum_{i,j} W_{ij}(\bm{y}_i^2 + \bm{y}_j^2)}{\sum_j \bm{y}_j^2 d_j} \quad \text{(By Lemma \ref{lem:diff-square})} \\
&= \frac{\sum_i \bm{y}_i^2\sum_j W_{ij} + \sum_j \bm{y}_j^2\sum_i W_{ij}}{\sum_j \bm{y}_j^2 d_j}\\ 
&= \frac{\sum_i \bm{y}_i^2 d_i + \sum_j \bm{y}_j^2 d_j}{\sum_j \bm{y}_j^2 d_j}\\
& = \frac{2\sum_j \bm{y}_j^2 d_j}{\sum_j \bm{y}_j^2 d_j} = 2 
\end{align*}
\end{proof}

  \bibliographystyle{elsarticle-num} 

\bibliography{references}

@article{shi2000normalized,
  title={Normalized cuts and image segmentation},
  author={Shi, Jianbo and Malik, Jitendra},
  journal={IEEE Transactions on pattern analysis and machine intelligence},
  volume={22},
  number={8},
  pages={888--905},
  year={2000},
  publisher={IEEE}
}

@article{von2007tutorial,
  title={A tutorial on spectral clustering},
  author={von Luxburg, Ulrike},
  journal={Statistics and computing},
  volume={17},
  number={4},
  pages={395--416},
  year={2007},
  publisher={Springer}
}

@book{chung1997spectral,
  title={Spectral graph theory},
  author={Chung, Fan RK},
  volume={92},
  year={1997},
  publisher={American Mathematical Soc.}
}

@article{hotelling1933analysis,
  title={Analysis of a complex of statistical variables into principal components.},
  author={Hotelling, H.},
  journal={Journal of Educational Psychology},
  volume={24},
  number={7},
  pages={498--520},
  year={1933},
  publisher={American Psychological Association}
}

@article{sorensen1992implicitly,
  title={Implicitly restarted Arnoldi methods},
  author={Sorensen, Danny C},
  journal={SIAM Journal on Matrix Analysis and Applications},
  volume={13},
  number={1},
  pages={357--385},
  year={1992},
  publisher={SIAM}
}

@book{saad1992numerical,
  title={Numerical methods for large eigenvalue problems},
  author={Saad, Yousef},
  year={1992},
  publisher={Manchester University Press}
}

@misc{amsmath,
  author =	 {{American Mathematical Society}},
  title =	 {User's Guide for the \texttt{amsmath} Package
                  (Version 2.0)},
  url =		 {ftp://ftp.ams.org/pub/tex/doc/amsmath/amsldoc.pdf},
  urldate =	 {2015-07-30},
  year =	 2002}

@misc{chierichetti2017,
      title={Fair Clustering Through Fairlets}, 
      author={Flavio Chierichetti and Ravi Kumar and Silvio Lattanzi and Sergei Vassilvitskii},
      year={2018},
      archivePrefix={arXiv},
      primaryClass={cs.LG},
      url={https://arxiv.org/abs/1802.05733}, 
}

@misc{kleindessner2019,
      title={Guarantees for Spectral Clustering with Fairness Constraints}, 
      author={Matthäus Kleindessner and Samira Samadi and Pranjal Awasthi and Jamie Morgenstern},
      year={2019},
      archivePrefix={arXiv},
      primaryClass={stat.ML},
      url={https://arxiv.org/abs/1901.08668}, 
}

@misc{wang2023,
      title={Scalable Spectral Clustering with Group Fairness Constraints}, 
      author={Ji Wang and Ding Lu and Ian Davidson and Zhaojun Bai},
      year={2023},
      archivePrefix={arXiv},
      primaryClass={cs.LG},
      url={https://arxiv.org/abs/2210.16435}, 
}

@article{Chhabra2021,
  author={Chhabra, Anshuman and Masalkovaitė, Karina and Mohapatra, Prasant},
  journal={IEEE Access}, 
  title={An Overview of Fairness in Clustering}, 
  year={2021},
  volume={9},
  number={},
  pages={130698-130720},
  doi={10.1109/ACCESS.2021.3114099}}

@misc{feldman2015,
      title={Certifying and removing disparate impact}, 
      author={Michael Feldman and Sorelle Friedler and John Moeller and Carlos Scheidegger and Suresh Venkatasubramanian},
      year={2015},
      primaryClass={stat.ML},
      url={https://arxiv.org/abs/1412.3756}, 
}

@article{OjedaRuiz2020,
title = {A fast constrained image segmentation algorithm},
journal = {Results in Applied Mathematics},
volume = {8},
pages = {100103},
year = {2020},
note = {Special Issue on Recent Advances in Computational Mathematics and Applications},
issn = {2590-0374},
url = {https://www.sciencedirect.com/science/article/pii/S2590037420300145},
author = {Iván Ojeda-Ruiz and Young-Ju Lee},
}

@inbook{Li_2023,
   title={Spectral Normalized-Cut Graph Partitioning with Fairness Constraints},
   ISBN={9781643684376},
   ISSN={1879-8314},
   url={http://dx.doi.org/10.3233/FAIA230416},
   DOI={10.3233/faia230416},
   booktitle={ECAI 2023},
   publisher={IOS Press},
   author={Li, Jia and Wang, Yanhao and Merchant, Arpit},
   year={2023}
}

@article{mitchell2021,
  title={Algorithmic fairness in the context of criminal justice: a review of the state-of-the-art},
  author={Mitchell, Emily and D'Amour, Alex and Kim, Sarah and Singh, Maninder},
  journal={arXiv preprint arXiv:2102.06739},
  year={2021}
}

@article{barocas2017,
  title={Fairness in machine learning: A survey},
  author={Barocas, Solon and Hardt, Moritz and Narayanan, Arvind},
  journal={arXiv preprint arXiv:1710.02271},
  year={2017}
}

@article{ng2002,
  title={On spectral clustering: Analysis and an algorithm},
  author={Ng, Andrew Y and Jordan, Michael I and Weiss, Yair},
  journal={Advances in neural information processing systems},
  volume={14},
  pages={849--856},
  year={2002}
}

@article{pothen1990spectral,
  title={Spectral methods for partitioning graphs and block matrices},
  author={Pothen, Alex and Simon, Horst D and Liou, Kang-Pu},
  journal={Linear algebra and its applications},
  volume={1},
  number={4},
  pages={1--22},
  year={1990},
  publisher={Elsevier}
}

@article{narayanan2020,
  title={Tutorial: 21 fairness definitions and their implications},
  author={Narayanan, Arvind},
  journal={Tutorial at the 2020 Conference on Fairness, Accountability, and Transparency (FAT* '20)},
  year={2020}
}

@article{lehoucq1998implicit,
  title={Implicitly restarted Arnoldi methods for the computation of eigenvalues},
  author={Lehoucq, Richard B and Sorensen, Danny C and Yang, Chao},
  journal={SIAM Journal on Matrix Analysis and Applications},
  volume={19},
  number={4},
  pages={1018--1032},
  year={1998},
  publisher={SIAM}
}

@book{boyd2004convex,
  title={Convex optimization},
  author={Boyd, Stephen and Vandenberghe, Lieven},
  year={2004},
  publisher={Cambridge university press}
}

@book{bertsekas1999nonlinear,
  title={Nonlinear programming},
  author={Bertsekas, Dimitri P},
  volume={2},
  year={1999},
  publisher={Athena Scientific}
}

@book{nocedal2006numerical,
  title={Numerical optimization},
  author={Nocedal, Jorge and Wright, Stephen},
  year={2006},
  publisher={Springer Science \& Business Media}
}

@book{horn2013matrix,
  title={Matrix analysis},
  author={Horn, Roger A and Johnson, Charles R},
  year={2013},
  publisher={Cambridge University Press}
}

@inproceedings{fairad,
  title={FairAD: Computationally Efficient Fair Graph Clustering via Algebraic Distance},
  author={Vuong, Minh Phu and Lee, Young-Ju and Ojeda-Ruiz, Iv{\'a}n and Lee, Chul-Ho},
  booktitle={Proceedings of the 34th ACM International Conference on Information and Knowledge Management},
  pages={2935--2944},
  year={2025}
}



\end{document}